\documentclass[lnbip,sechang,a4paper]{svmultln}
\usepackage[sectionbib]{natbib}
\usepackage{diagbox}
\usepackage[section]{placeins}
\usepackage{authblk}
\usepackage{tabularx}
\usepackage{booktabs}
\usepackage[marginal]{footmisc}
\usepackage{amsmath}

\usepackage{amsthm}
\usepackage{algorithm}
\usepackage{textcomp}
\usepackage{algorithmic}
\usepackage{amssymb}
\usepackage{graphicx}
\usepackage{url}
\usepackage{hyperref}
\begin{document}
\def\QEDclosed{\mbox{\rule[0pt]{1.3ex}{1.3ex}}} 
\mainmatter
\def\widebar{\accentset{{\cc@style\underline{\mskip10mu}}}}
\title{THORS: An Efficient Approach for Making Classifiers Cost-sensitive}

\titlerunning{THORS: An Efficient Approach for Making Classifiers Cost-sensitive}
\author{Ye Tian
\qquad
Weiping Zhang
\thanks{Corresponding author. Email: zwp@ustc.edu.cn.}}
\authorrunning{Ye Tian, Weiping Zhang}

\institute{School of Data Science, Department of Statistics and finance\\
University of Science and Technology of China, Hefei, 230026, P.R.China\\
}

\tocauthor{Authors' Instructions}
\maketitle

\begin{abstract}
In this paper, we propose an effective \textit{TH}resholding method based on \textit{OR}der \textit{S}tatistic, called THORS, to convert an arbitrary scoring-type classifier, which can induce a continuous cumulative distribution function of the score, into a cost-sensitive one. The procedure, uses order statistic to find an optimal threshold for classification, requiring almost no knowledge of classifiers itself. Unlike common data-driven methods,  we analytically show that THORS has theoretical guaranteed performance,  theoretical bounds for the costs and lower time complexity. Coupled with empirical results on several real-world data sets, we argue that THORS is the preferred cost-sensitive technique.
\end{abstract}

\section{Introduction}
Classification is one of the most important tasks in machine learning and data mining. A classifier is usually trained from a set of training instances with discrete and finite class labels to predict the class labels of new instances. Many effective classification algorithms have been developed, such as linear algorithms \citep{balakrishnama1998linear}, neural network \citep{krizhevsky2012imagenet}, Bayesian classifier \citep{mccallum1998comparison}, decision tree \citep{safavian1991survey} and instance-based classifiers \citep{sheng2009cost}. However, most of the currently-available algorithms implicitly assume that all errors are equally costly,  which may be inadequate for problems with various misclassification costs \citep{domingos1999metacost}. In many KDD applications, costs are often different for different types of errors. For example, in fraud detection, undetected frauds with high transaction amounts are obviously more costly \citep{fan1999adacost, zonneveldt2010bayesian}. Besides, in medical diagnosis, it's far more serious to diagnose someone with a life-threatening disease as healthy than diagnose someone healthy as ill \citep{tong2018neyman, viaene2005cost}.
As a result, a lot of work related to cost-sensitive learning has been done recently and they seek to minimize total misclassification costs rather than error rate. \cite{sheng2009cost} divide the existing cost-sensitive algorithms into two categories: One is to design cost-sensitive classifiers that are cost-sensitive in themselves \citep{chai2004test, drummond2000exploiting, turney1994cost} and the other is to design a ``wrapper'' that converts any existing cost-insensitive classifiers into cost-sensitive ones, called cost-sensitive meta-learning or wrapper method \citep{domingos1999metacost, elkan2001foundations, fan1999adacost, sheng2006thresholding, sun2007cost, ting1998inducing, witten2016data,  zadrozny2001learning, zhao2008instance}. Our work belongs to the second category.

The wrapper method can be further categorized as \textit{thresholding}, \textit{sampling} and \textit{weighting} \citep{sheng2006thresholding}. \textit{Thresholding} finds the best probability (or other scores) which minimizes the total misclassification cost from the training instances as the threshold, and uses it to predict the class label of test instances. Metacost  \citep{domingos1999metacost} is a thresholding method. It firstly learns a classifier on each multiple bootstrap replicates of the training set to obtain reliable probability estimations of training instances by voting, and then relabels training instances according to their estimated minimal cost classes, finally uses the relabeled training instances to build a cost-sensitive classifier. Metacost can be applied to multi-class problems and to arbitrary cost matrices. Instead of relabeling the training instances, Cost Sensitive Classifier (CSC) \citep{witten2016data} relabels the test instances. \cite{elkan2001foundations} obtains the theoretical threshold for making optimal cost-sensitive classification in two-class case, and argues that changing the balance of negative and positive training instances has little effect on the classifier learned by standard decision tree learning methods \citep{elkan2001foundations}. Noting that estimating the probability accurately is crucial in thresholding-based meta-learning methods, \cite{zadrozny2001learning} propose several methods to improve the calibration of probability estimates, while \cite{sheng2006thresholding} develop an empirical thresholding method which does not require accurate estimation of probabilities \cite{sheng2006thresholding}. Alternatively, \textit{sampling} method modifies the class distribution of the training data, then applies cost-insensitive classifiers on the sampled data directly. \textit{Weighting} \citep{ting1998inducing} can be viewed as a sampling method, in which different types of instances in the training data are weighted according to the misclassification costs during classifier learning, such that the classifier strives to make fewer errors of the more costly type, resulting in lower overall cost. Methods based on weighting include \cite{sun2007cost, zadrozny2003cost, zhao2008instance}.

While a plethora of cost-sensitive methods have been investigated, several issues still remain to be addressed. Firstly, empirical comparisons instead of theoretical properties are frequently reported for the existing algorithms. The theoretical error bounds for the costs are less explored for most cost-sensitive classifiers. While common practices that directly limit the empirical false negative rate to no more than specified level show that the resulting classifiers are likely to have a much larger false negative rate.
Secondly, the time complexity of a cost-sensitive classifier is an important concern since existing re-sampling and weighting methods are computationally more involved. We propose a \textit{TH}resholding method based on \textit{OR}der \textit{S}tatistic method, named THORS,  to convert an arbitrary scoring-type classifier, which induces a continuous cumulative distribution function of scores, into a cost-sensitive one. It uses the order statistics of classification scores on validation set to build an optimal classification threshold, instead of estimating the optimal threshold probability. We analytically show that the existence of the optimal threshold and the error bounds of costs because it does not need to re-sample the training instances. THORS also has lower time complexity compared with the existing popular cost-sensitive classifiers. It usually leads to smaller total cost than empirical approaches, Metacost, and Cost-proportionate Rejection Sampling (CRS) methods, even on the heavily imbalanced dataset.

The remainder of the paper is organized as follows. In Section 2 and 3, detailed THORS algorithm and its theoretical properties are described. In Section 4, we evaluate our method on three real data sets against other existing methods, such as theoretical thresholding, empirical thresholding, and meta-learning algorithms including Metacost and CRS. We conclude this paper by summarizing the main findings and outlining future research in Section 5.

\section{The THORS Algorithm}
The theory of cost-sensitive learning presented by \citep{elkan2001foundations, sheng2009cost} describes how the misclassification cost plays its role in various related cost-sensitive algorithms. Without loss of generality, we assume binary classification (i.e., positive and negative class) in this paper, where the objective is to predict the value of a binary-dependent variable, referred to as the class, based on a vector of independent variables (also called attributes or features). In THORS, the full data is divided into training set and validation set, which are used to train the classifier and find optimal threshold, respectively.  Then the scoring function $f$ is trained using the training set and assigns a classification score to an observation $x$, the class label is predicted by whether its score $f(x)$ is larger than a threshold $c$. Most popular classification methods are of this type, including SVMs, Na\"{i}ve Bayes, logistic regression and neural networks. The classification scores can be strict probabilities or uncalibrated numeric values as long as a higher score indicates a higher probability of an observation belonging to the positive class. The optimal threshold $c$ is selected as the minimizer of estimated misclassification expected cost in (\ref{escost}) on validation set, which usually uses cross-validation to search the best threshold value from the training dataset.

\subsection{Cost Matrix}
In cost-sensitive learning, the costs of false positive (actual negative but predicted as
positive; denoted as \textit{FP}), false negative (\textit{FN}), true positive (\textit{TP}) and true negative (\textit{TN}) can be
given in a cost matrix, as shown in Table \ref{tab1}, where  $C(j,k)$ denotes the cost of classifying the instance be class $k$ when it is actually in class $j$ ($j,k=0,1$). $C(j,j)$ (\textit{TP} and \textit{TN}) is usually regarded as the ``benefit"(i.e., negated cost) when the instance is predicted correctly.

\begin{table}
\centering
\caption{Cost Matrix of All the Instances in A Binary Classification\label{tab1}}
\begin{tabular}{|c|c|c|}
\hline
\diagbox{Actual}{Predict} &Negative &Positive \\
\hline
Negative &$C(0,0)$	&$C(0,1)$  \\
\hline
Positive &$C(1,0)$ 	&$C(1,1)$ \\
\hline
\end{tabular}
\end{table}
Usually, the minority class is regarded as the positive class, and it is often more expensive to misclassify an actual positive instance into negative, than an actual negative instance into positive. That is, the value of \textit{FN} or $C(1,0)$ is usually larger than that of \textit{FP} or $C(0,1)$. This is true for the fraud detection example mentioned earlier (fraud transaction is usually rare, but predicting an actual fraud transaction as negative is usually more costly) and the medical diagnosis example. Without loss of generality, we consider the case where $C(0,0)=C(1,1)=0$ and $\beta=C(0,1)/C(1,0)$ for some $0 < \beta < 1$.
Under the case we consider, the expected cost for classifying $n$ instances is
\begin{equation}\label{cost}
C=\sum_{i=1}^n[\pi_0C(0,1)P_i(1|0)+\pi_1C(1,0)P_i(0|1)],
\end{equation}
where $\pi_j$ is the marginal probability of class $j$ and $P_i(1|0)$ is the probability of classifying instance $i$ into class 1 when it actually in class 0, which is called the \textit{False Positive Rate} (\textit{FPR}); $P_i(0|1)$ is the probability of classifying instance $i$ into class 0 when it actually in class 1, and is called \textit{False Negative Rate} (\textit{FNR}). Obviously, the classification error is a weighted sum of \textit{FNR} and \textit{FPR}. In practice, both $\pi_j$ and $P_i(j|k) (j,k=0,1)$ are unknown. They should be estimated on validation set and plugged into (\ref{cost}), where the corresponding estimations are $\widehat{\pi}_j$ and $\widehat{P}_i(j|k)$ $(j,k=0,1)$ and we have the estimated expected cost
\begin{equation}\label{escost}
\widehat{C} = \sum_{i=1}^n[\widehat{\pi}_0C(0,1)\widehat{P}_i(1|0)+\widehat{\pi}_1C(1,0)\widehat{P}_i(0|1)].
\end{equation}

\subsection{Thresholding on Order Statistic}

Firstly we split the full data into training set and validation set. The training set will be used for training the classifier and the validation set will help us to find the optimal threshold. Intuitively, the optimal threshold splits the classification scores into two parts and produces minimal estimated expected misclassification cost in (\ref{escost}) on the validation set. Our proposed THORS algorithm picks one of the scores as the optimal threshold which is an order statistic of scores to obtain a minimal total misclassification cost on validation set. Recently, \cite{tong2018neyman} show that a binary classifier by choosing order statistic as an optimal threshold guarantees the desired high-probability control of type I error $P_i(1|0)$. We prove that the errors and total misclassification cost by THORS algorithm are similarly bounded theoretically.

For a given binary classifier $f$, a new instance $i$ with the predictor vector $x_{i}$ is predicted as class positive (1) if $f(x_{i}) > c^*$, otherwise it will be classified into class negative (0), that is,
\begin{equation}\label{classification function}
\psi(x_i) =
\left\{
\begin{array}{rcl}
0,      &      &f(x_{i}) \leq c^*,\\
1,     &      &f(x_{i}) > c^*.\\
\end{array}
\right.
\end{equation}
where $c^*$ is the optimal threshold needing to be learned from the validation set. The score function $f$ is learned from the training data set. $f(x_i)$ is the score of  instance $i$ with feature vector $x_i$ under the classifier $\psi$.
Applying the classifier $f$ to the validation set of size $n_v$ with $n_{0}$ class 0 and $n_{1}$ class 1 instances, we obtain $n_v$ sorted scores as $T_{(1)} \leq T_{(2)} \leq ... \leq T_{(n_{v})}$, and scores in class 0 and class 1 are denoted as $T^{0}_{(1)} \leq T^{0}_{(2)} \leq ... \leq T^{0}_{(n_{0})}$ and $T^{1}_{(1)} \leq T^{1}_{(2)} \leq ... \leq T^{1}_{(n_{1})}$, respectively. Then we have the following theorem for selecting the optimal threshold.

\begin{theorem}\label{th1}
Let $k_{0}=\arg\max\{1\leq i\leq n_{0}:T^{0}_{(i)} \leq T_{(k)}\}$ and $k_{1}=\arg\max\{1\leq i\leq n_{1}:T^{1}_{(i)} \leq T_{(k)}\}$. $\widehat{\pi}_0=\frac{n_0}{n_v}$ and $\widehat{\pi}_1=\frac{n_1}{n_v}$ are estimations of the marginal probability of class 0 and 1 samples in the population. Then the optimal threshold can be decided by minimizing the following part of estimated expected cost (\ref{escost}) on validation set,
\begin{equation}\label{ecost}
  \frac{\widehat{\pi}_1k_1}{n_1} - \frac{\beta\widehat{\pi}_0k_0}{n_0}.
\end{equation}
That is, the optimal threshold $c^*=T_{(k^*)}$, where
\begin{equation}\label{minimal}
k^* = \arg\min\limits_{1\leq k\leq n_v}\Big\{\frac{\widehat{\pi}_1k_1}{n_1} - \frac{\beta\widehat{\pi}_0k_0}{n_0}\Big\}.
\end{equation}
\end{theorem}

It is worth noting that Theorem \ref{th1} does not rely on any distributional assumptions or on base algorithm characteristics.
Besides, from Theorem \ref{th1} it is known that if the rank $k$ is fixed then $k_0$ and $k_1$ become fixed as well. Therefore we can estimate (\ref{ecost}) for each choice of $k$ on validation set and choose the optimal one to make (\ref{minimal}) minimal in nearly linear time, which is much more efficient than most of the empirical methods.

We summarize the THORS method as the following \textbf{Algorithm 1}.
\vskip 5pt
\begin{algorithm}
\caption{THORS algorithm}
\noindent\hspace*{0.02in}{\bf Input:}
\\
\hspace*{0.2in} $D_v$: Validation set with size $n_v$, including $n_{0}$ class 0 and $n_{1}$ class 1 samples\\
\hspace*{0.2in} $f$: Classification score function \\
\hspace*{0.02in}{\bf Output:}
\\
\hspace*{0.2in} $c^*$: Estimated optimal threshold \\
\hspace*{0.02in}{\bf Other parameters:}
\\
\hspace*{0.2in} $\widehat{\pi}_0$: $\frac{n_{0}}{n_v}$ \\
\hspace*{0.2in} $\widehat{\pi}_1$: $\frac{n_{1}}{n_v}$ \\
\hspace*{0.2in} $x_{i}$: Feature vector of instance $i$\\
\hspace*{0.2in} $x_{(i)}$: Feature vector of $i$-th order statistic\\
\hspace*{0.2in} $Y_{(i)}$: True class of $i$-th order statistic\\
\begin{algorithmic}[1]
\STATE $T_{1},...,T_{n_v} \longleftarrow f(x_1),...,f(x_{n_v})$
\STATE $T_{(1)},...,T_{(n_v)}$ $\longleftarrow$ sort $T_{1},...,T_{n_v}$
\STATE $l_0, l_1 \longleftarrow 0$
\FOR{$k$ in $\{1,2,...,n_v\}$}
    \IF{$Y_{(k)}=0$}
        \STATE $l_0 \longleftarrow l_0+1$
    \ELSE
        \STATE $l_1 \longleftarrow l_1+1$
    \ENDIF
    \STATE $k_0 \longleftarrow l_0, k_1 \longleftarrow l_1$
    \STATE $C(k) \longleftarrow \frac{\widehat{\pi}_1 k_1}{n_1} - \frac{\beta \widehat{\pi}_0 k_0}{n_0}$
\ENDFOR
\STATE $k^* \longleftarrow \mathop{argmin}\limits_{k}C(k)$
\STATE threshold $c^* \longleftarrow T_{(k^*)}$
\RETURN $c^*$
\end{algorithmic}
\end{algorithm}

\section{Properties}
For the THORS method described in the previous section, we can show that its \textit{FPR}, \textit{FNR} and expected misclassification cost are bounded by the following theorems. These properties hold for any score-type classifier that introduces a continuous cumulative distribution function of the score. For discontinuous cases, properties can be approximately correct.

\subsection{Theoretical Upper Bound}

Under Algorithm 1, let $T_i=f(x_i)$ be the score for  a new instance $i$ with class label $c_i$ and feature $x_i$ under the classifier $\psi(x_i)$, the corresponding \textit{FNR} and \textit{FPR} are $\alpha_{i1}=P(\psi(x_i)=0|c_i=1)$ and $\alpha_{i0}=P(\psi(x_i)=1|c_i=0)$, respectively. Due to the randomness of the order statistic and the relationship between order statistic of scores and the threshold we choose, here both two types of error rates are random variables. Noting that $c^*=T_{(k^*)}$ is fixed in the computation of conditional probabilities, for the bounds of cumulative distribution functions of \textit{FNR} and \textit{FPR}, we have the following result.

\begin{theorem}\label{errbds}
Let $\alpha_{i1}$ and $\alpha_{i0}$ be \textit{FNR} and \textit{FPR} of classifying new instance $i$ with true class label $c_i$ using the classifier $\psi(x_i)$ under THORS algorithm, respectively. When the distribution of the score is continuous, for any $x \in (0,1)$, we have
\begin{align}
\sum_{j=n_0-k_0+1}^{n_0}\binom{n_{0}}{j}x^{j}(1-x)^{n_{0}-j} \leq &P(\alpha_{i0} \leq x) \leq \sum_{j=n_0-k_0}^{n_0}\binom{n_{0}}{j}x^{j}(1-x)^{n_{0}-j},\label{cdfbd1}\\
\sum_{j=k_{1}+1}^{n_{1}}\binom{n_{1}}{j}x^{j}(1-x)^{n_{1}-j} \leq &P(\alpha_{i1} \leq x) \leq \sum_{j=k_{1}}^{n_{1}}\binom{n_{1}}{j}x^{j}(1-x)^{n_{1}-j}.\label{cdfbd2}
\end{align}
where the numbers of class 0 and 1 in validation set are denoted as $n_0$ and $n_1$ and definitions of $k_0$, $k_1$ are the same as Theorem \ref{th1}. That is, the bounds of cumulative distribution function of \textit{FNR} and \textit{FPR} are only dependent on $n_i, k_i (i=0,1)$, which will be fixed for specified validation set.
\end{theorem}

According to Theorem \ref{errbds}, we have the following high probability upper bound for the expected cost on $n_{te}$ new instances.

\begin{theorem}\label{costbd1}
Let $C$ be the expected misclassification cost of the THORS algorithm on $n_{te}$ new instances, then there exist some constant $C^*, M, \sigma$ such that
\begin{equation}
P(C \leq C^{*} + t\sigma) \geq 1 - exp\Bigg\{-\frac{t^2}{2+\frac{2M}{3\sigma}t} \Bigg\},
\end{equation}
where $C^*, M$ and $\sigma$ satisfy
\begin{align}
C^{*}& = n_{te}C(1,0)\Big[\beta \pi_0\Big(\frac{n_0-k_0+1}{n_0+1}\Big) + \pi_1\Big(\frac{k_1+1}{n_1+1}\Big)\Big],\\
M &= max\Big\{\pi_1 C(1,0)max\Big(\frac{k_1+1}{n_1+1},1-\frac{k_1+1}{n_1+1}\Big), \nonumber\\
&\qquad\beta \pi_0 C(1,0)max\Big(\frac{n_0-k_0+1}{n_0+1},1-\frac{n_0-k_0+1}{n_0+1}\Big)\Big\},\\
\sigma^2 & = n_{te}[C(1,0)]^2\Big[\pi_1^2 \frac{(k_1+1)(n_1-k_1)}{(n_1+1)^2(n_1+2)}+(\beta \pi_0)^2 \frac{k_0(n_0-k_0+1)}{(n_0+1)^2(n_0+2)}\Big].
\end{align}
Here $n_{te}$ represents the number of new instances and $n_0, n_1, k_0, k_1$ are defined the same as Theorem \ref{errbds}.
\end{theorem}

\subsection{Time Complexity}

For empirical method the time complexity is $a_3\cdot \frac{range}{pc}\cdot n_v$, where \textit{range} is the searching range for optimal threshold and $pc$ is the searching precision. And $a_3$ is related to the type of classifier itself and $n_v$ is the size of validation set used for estimate the threshold \citep{sheng2006thresholding}. And the time complexity of Metacost method \citep{domingos1999metacost} is $a_4 \cdot m \cdot n_{tr}$ in which $m$ is the number of resampled instances to generate and $n_{tr}$ is the size of training set. $a_4$ is the constant related to classifier and the sampling algorithm. For Cost-proportionate Rejection Sampling (CRS) \citep{witten2016data} that is $a_5 \cdot \frac{Z}{c}\cdot n_{tr} $, where $\frac{c}{Z}$ is acceptance probability and both $c$, $Z$ are defined in \citep{witten2016data}. $a_5$ is a constant related to the performance of classifier itself. To obtain a good performance, $pc,m$ are needed to be chosen not so small, leading to high time-complexity. However, our numerical studies in the next section show that THORS can find the optimal threshold more faster than empirical method and Metacost in a short time. Although the time complexity of CRS is also small, the performance of it is much worse than THORS for real data, which will be shown by the following section.

\begin{theorem}\label{time}
The time-complexity of THORS algorithm is $a_1n_vlog\,n_v + a_2n_v$, where $a_1$, $a_2$ are constants related to the sorting method and the classifier itself respectively ($a_1$ is very small). And $n_v$ is the size of validation set.
\end{theorem}

\subsection{Short Theoretical Bounds of Misclassification Cost Expectation}
We will see that the range of expected misclassification cost will approximate to short bounds under a high probability when the validation set size $n_v$ is large enough.

\begin{theorem}\label{costbd2}
The expected misclassification cost $C$ is bounded by some constants, that is
\begin{align}
&P(C_1 - C_\varepsilon \leq C \leq C_2 + C_\varepsilon)\nonumber\\
&\geq 1 - n_{te}\Big[exp\Big\{-2\Big[\varepsilon+\frac{n_1-k_1}{n_1(n_1+1)}\Big]^2 n_1\Big\} +
exp\Big\{-2\Big[\varepsilon+\frac{k_1}{n_1(n_1+1)}\Big]^2 n_1\Big\} \nonumber\\
&+ exp\Big\{-2\Big[\varepsilon+\frac{k_0}{n_0(n_0+1)}\Big]^2 n_0\Big\} +
exp\Big\{-2\Big[\varepsilon+\frac{n_0-k_0}{n_0(n_0+1)}\Big]^2 n_0\Big\}\Big]\nonumber\\
&\geq 1 - O(exp\{-\varepsilon^2 n_v\}).
\end{align}
Particularly, the cost $C$ is upper-bounded,
\begin{align}
&P(C \leq C_2 + C_\varepsilon)\nonumber\\
&\geq 1 - n_{te}\Big[exp\Big\{-2\Big[\varepsilon+\frac{k_1}{n_1(n_1+1)}\Big]^2 n_1\Big\} +exp\Big\{-2\Big[\varepsilon+\frac{n_0-k_0}{n_0(n_0+1)}\Big]^2 n_0\Big\}\Big] \label{cost.ub1}\\
&\geq 1 - O(exp\{-\varepsilon^2 n_v\}).\label{cost.ub2}
\end{align}
where
\begin{align}
C_1 &= n_{te}C(1,0)\Big[\pi_1\Big(\frac{k_1}{n_1+1}\Big) + \beta \pi_0\Big(\frac{n_0-k_0}{n_0+1}\Big)\Big],\\
C_2 &= n_{te}C(1,0)\Big[\pi_1\Big(\frac{k_1+1}{n_1+1}\Big) + \beta \pi_0\Big(\frac{n_0-k_0+1}{n_0+1}\Big)\Big],\\
C_\varepsilon &= n_{te}C(1,0)(\pi_1 + \beta \pi_0)\varepsilon,
\end{align}
where $n_{te}$ is the size of test set.
\end{theorem}
When $n_v$ is large enough, $\varepsilon$ can be taken as a small number and then the length of the interval $[C_1 - C_\varepsilon, C_2 + C_\varepsilon]$ could be relatively small and the upper bound $C_2 + C_\varepsilon$ would approximate to the least upper bound. \\ Furthermore, the following theorem is useful for estimating the size of a validation set to control expected cost with given precision.

\begin{theorem}\label{sample.size}
Let $n_0$ and $n_1$ be the size of class 0 and class 1 instances in validation set, $\widehat{g}_{n_v}(x)$ and $g(x)$ be empirical expected cost for each sample on validation set $\frac{\widehat{C}}{n_v}$ and the expected cost for each sample in population when threshold is $x$, respectively. Denote $\widehat{\alpha}_{i0}=1-\frac{k_0}{n_0}, \widehat{\alpha}_{i1}=\frac{k_1}{n_1}$ as the empirical \textit{FPR} and \textit{FNR} on validation set, and $\alpha_{i0}, \alpha_{i1}$ as \textit{FPR} and \textit{FNR} in population, based on optimal threshold obtaining by \textit{THORS}.  If the following four assumptions hold:\\
(A1) $g(x)$ is continuous;\\
(A2) $\widehat{g}_{n_v}(x) \xrightarrow{p} g(x)$, uniformly for fixed $x$, as $n_v \longrightarrow \infty$;\\
(A3) $\widehat{\alpha}_{i0}=\frac{n_0-k_0}{n_0} \xrightarrow{p} \alpha_{i0}, \widehat{\alpha}_{i1}=\frac{k_1}{n_1} \xrightarrow{p} \alpha_{i1}$, as $n_v \longrightarrow \infty$;\\
(A4) $g(x)$ has a unique minimal point at which $g$ reaches its minimum,\\
where $k_0$ and $k_1$ are defined as Theorem \ref{th1},  then there exist
\begin{align}
\frac{n_0-k_0}{n_0+1} &\xrightarrow{p}\,\,constant,\label{ratio.const1}\\
\frac{k_1}{n_1+1} &\xrightarrow{p}\,\,constant,\label{ratio.const2}
\end{align}
as $n_v \longrightarrow \infty$.
\end{theorem}

It is easy to see that assumption (A1) will always hold for a scoring-type classifier. Besides, it's natural to assume that the empirical cost or error rate on validation set converges to population cost or error rate in probability. Thus assumption (A2) and (A3) follow reasonably. The assumption (A4) is related to the problem itself, most well-defined problems satisfy such assumption. This assumption guarantees that our solution can be stable. For (\ref{ratio.const1}) and (\ref{ratio.const2}), the ratios $\frac{k_1}{n_1+1}$ and $\frac{n_0-k_0}{n_0+1}$ can be regarded as constant as $n_v$ goes to infinity in practice if these assumptions hold. Then if $C_\varepsilon$ is given, then $\varepsilon$ can be obtained. Let $n_1 = \widehat{\pi}_1 n_v$ and $n_0 = \widehat{\pi}_0 n_v$ approximately, $n_v$ can be solved from (\ref{cost.ub1}) if we let (\ref{cost.ub1}) be a fixed number between 0 and 1. The size $n_v$ is a conservative estimation of the minimal size to control the cost under a fixed probability. We will calculate this for specific datasets in the following section.

\section{Case Studies}
In this section, we will focus on three real datasets from UCI Machine Learning Repository, in which their imbalance rates decrease from 59:1 to 1.84:1. THORS will be applied on them and results will be compared with other thresholding and meta-learning methods, including both total cost on test set and average running time on the 8 GB RAM laptop with Intel\textsuperscript{\tiny\textregistered} Core\texttrademark\, i5-6300U CPU. The results show that THORS outperforms the alternatives even when the data set is heavily imbalanced.
\subsection{Scania Trucks Data}
We implement THORS on a Scania trucks dataset of UCI Machine Learning Repository. It records 60,000 component failures for a specific component of the APS system. And these samples fall into two categories: 1,000 failures for a specific component of the APS system and 59,000 ones not related to APS. To formulate this problem into a cost-sensitive classification one, we denote APS related failures as class 1 (positive) and unrelated ones as class 0 (negative). In this case, \textit{FP} refers to the cost that an unnecessary check needs to be done by a mechanic at a workshop, while \textit{FN} refers to the cost of missing a faulty truck, which may cause a breakdown. And costs for \textit{FN} and \textit{FP} are set as 500 and 1, respectively. The imbalance rate here is 59:1, which is a heavily imbalanced case. There are 171 attributes for each observation. We pre-process the original data before starting classification. We choose 10 prominent attributes used for training the classifier through ANOVA F-value for the provided samples. Then the data is divided into three parts, which are training set, validation set and test set. Base classifiers we choose for this problem are logistic regression with cost weighting (Logit), decision tree (DT, combined with Adaboosting), Na\"ive Bayes (NB) and linear discriminant analysis (LDA). Among 60,000 instances, 24,000 observations are used in training base classifier and 24,000 ones are applied to choose the optimal threshold. Remained 12,000 observations are divided into test set. As a comparison, besides our algorithm we also choose two other thresholding methods, including empirical method \citep{sheng2006thresholding} and theoretical thresholding \citep{elkan2001foundations}. Also we compare our results with other meta-learning methods such as Metacost \citep{domingos1999metacost} and CRS (Cost-proportionate Rejection Sampling) \citep{zadrozny2003cost}. Null model (default base classifier) is used as the baseline. Each algorithm is run for 20 times. Average costs with corresponding standard variances, comparison of performance, and the running time for each algorithm on each classifier are reported.

\begin{table*}
\centering
\caption{Average Costs and Standard Deviations for Each Algorithm on Each Classifier for Trucks Data\label{tab2}}
\begin{tabular}{|c|c|c|c|c|c|c|}
\hline
&THORS &Null &Theoretical &Empirical &Metacost &CRS\\
\hline
Logit &$12169 \pm 1277$	&$115099 \pm 1118$ &$117962 \pm 126$ &$15571 \pm 8354$ &$120671 \pm 2743$ &$116602 \pm 885$
 \\
\hline
DT &$12639 \pm 1228$ &$57669 \pm 4836$ &$117962 \pm 126$ &$17613 \pm 14495$ &$117962 \pm 126$ &$32092 \pm 3772$
 \\
\hline
NB &$12550 \pm 1664$ &$17592 \pm 2282$ &$16883 \pm 2173$ &$16934 \pm 2446$ &$12396 \pm 1432$ &$17640 \pm 2819$
 \\
\hline
LDA &$13025 \pm 1484$ &$39097 \pm 3656$ &$29770 \pm 2905$ &$33598 \pm 6477$ &$33802 \pm 3037$ &$34464 \pm 3249$
\\
\hline
\end{tabular}
\end{table*}

\begin{figure}[htb]
\centering
  \begin{tabular}{@{}cccc@{}}
    \includegraphics[width=.450\textwidth]{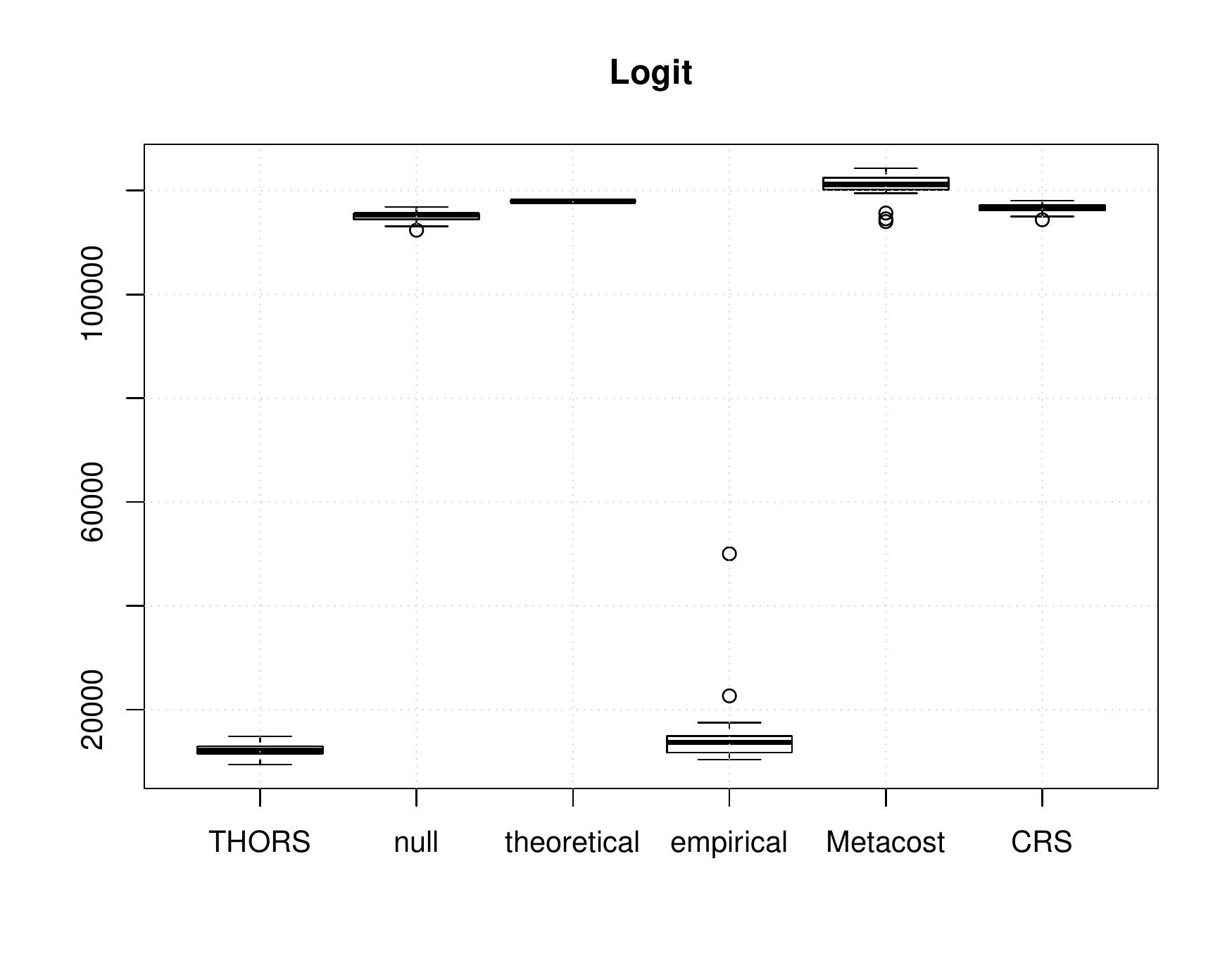} &
    \includegraphics[width=.450\textwidth]{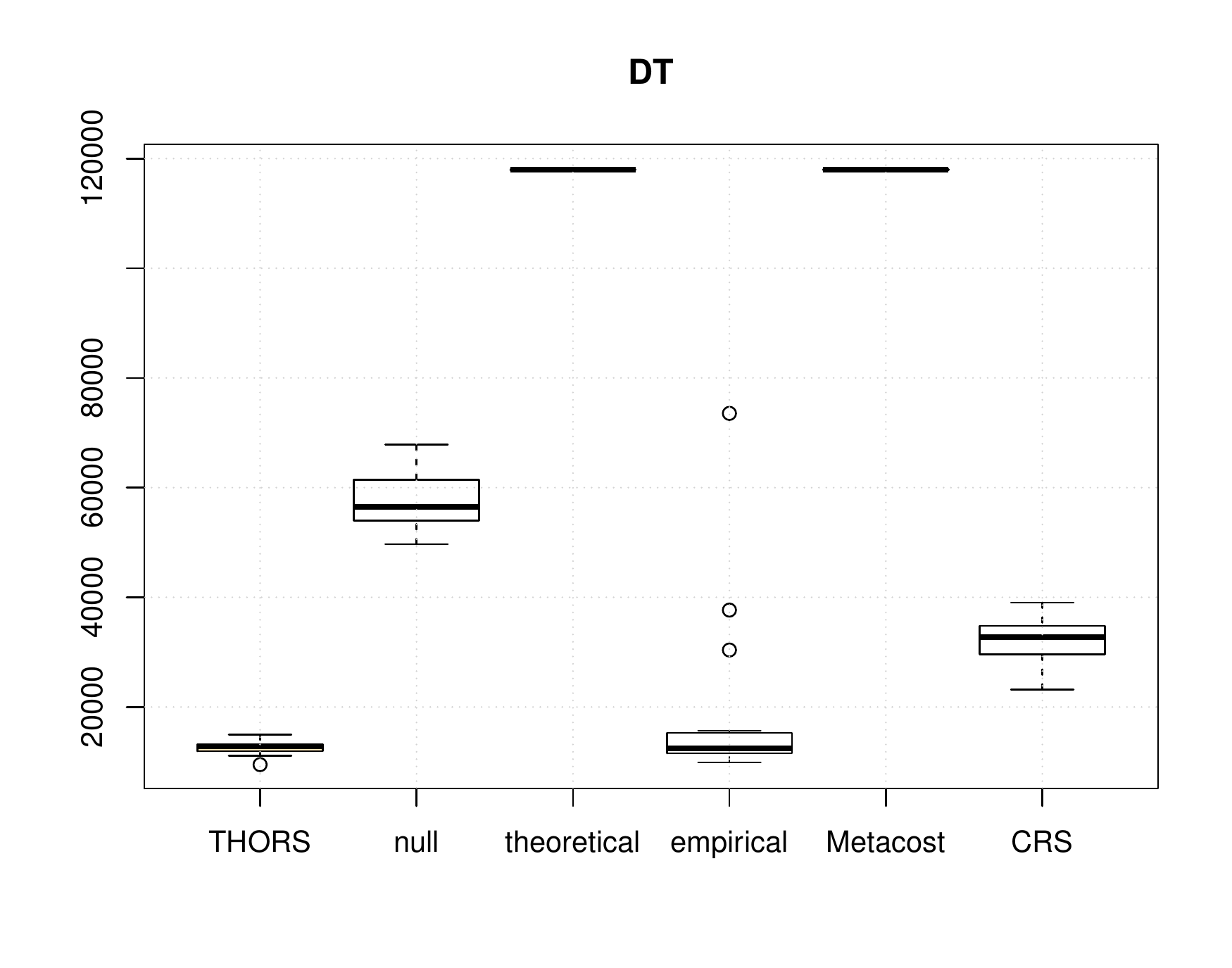}   \\
    \includegraphics[width=.450\textwidth]{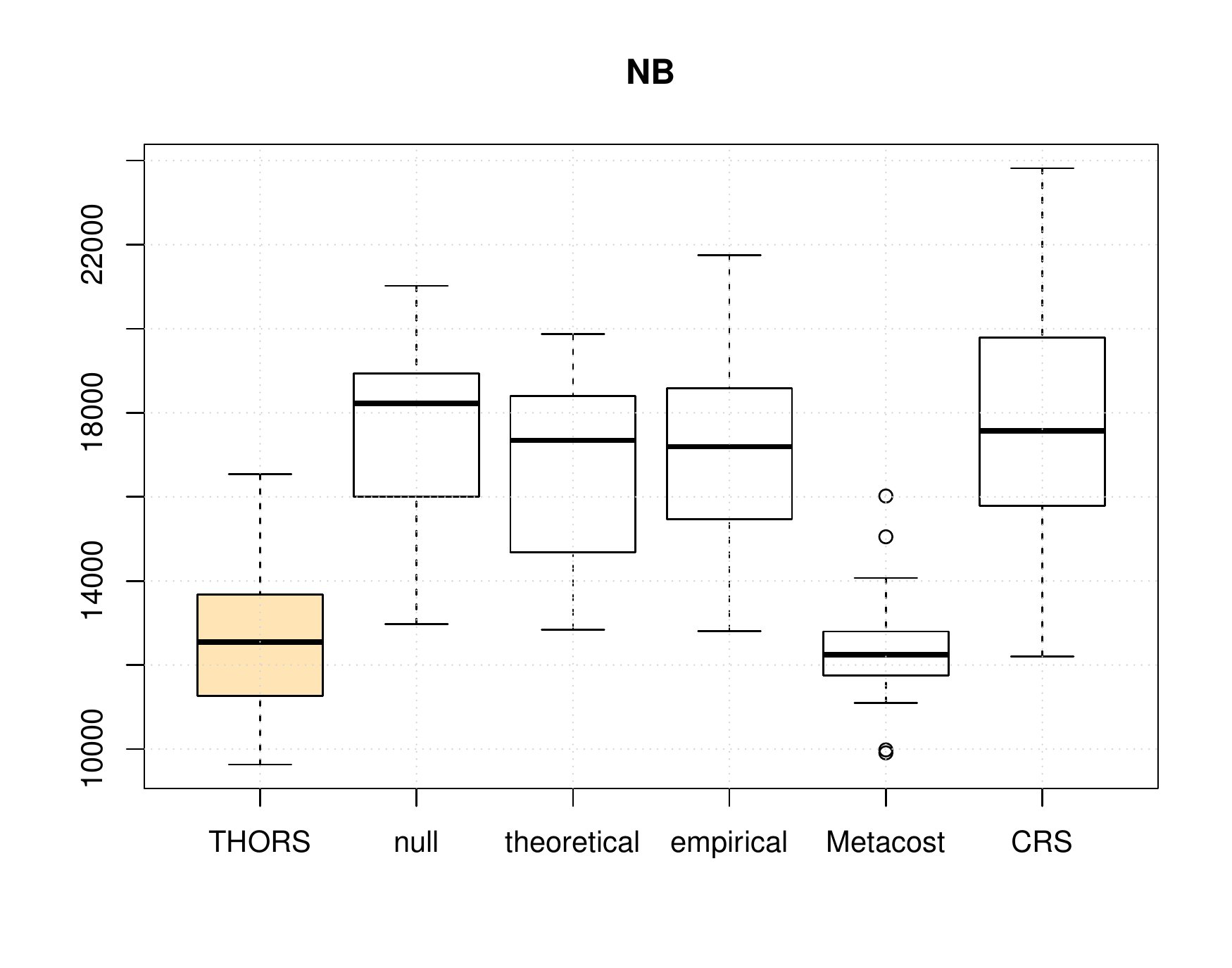} &
    \includegraphics[width=.450\textwidth]{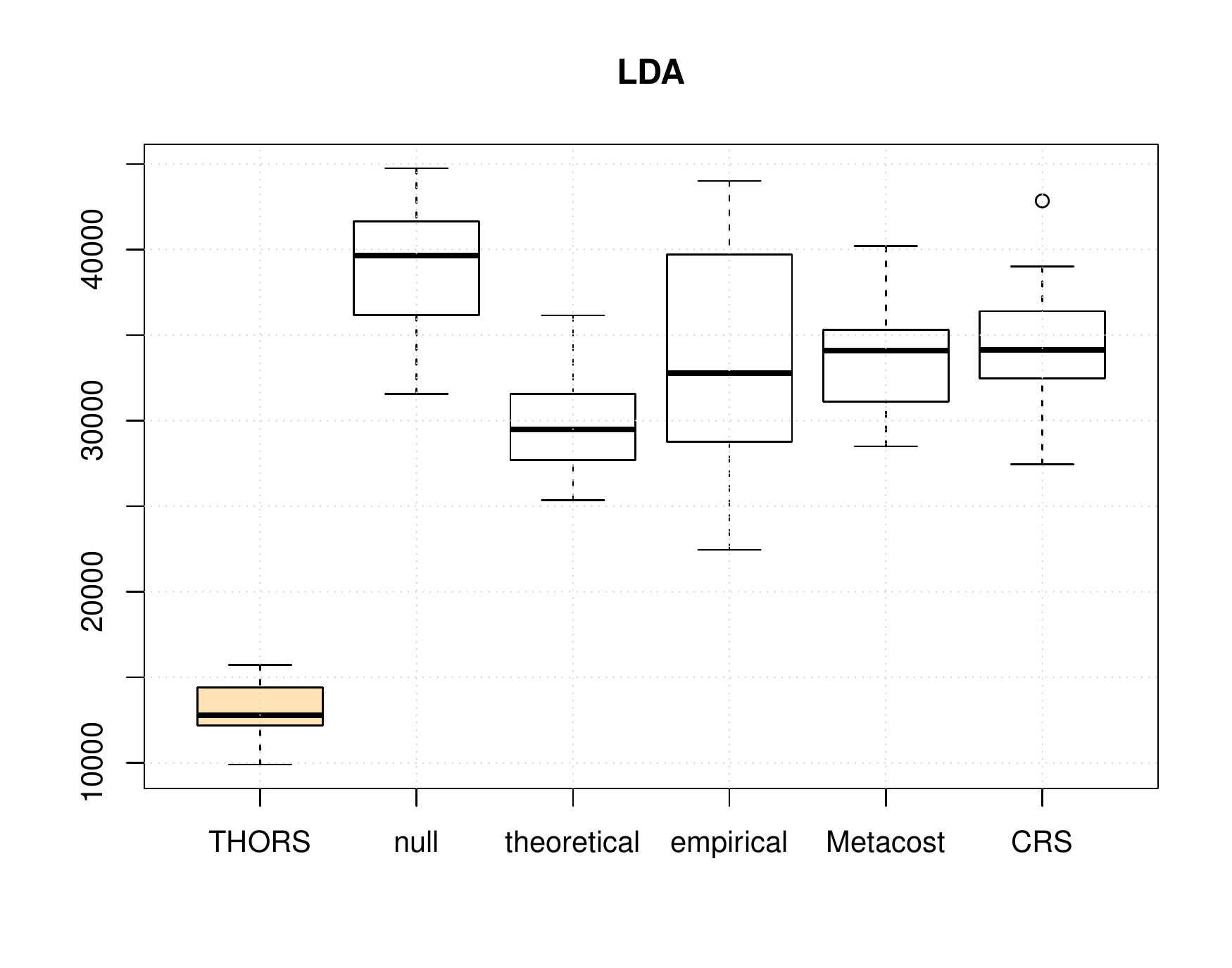}   \\
  \end{tabular}
  \caption{Box-plots of Costs on Test Data by Different Methods for Trucks Data\label{fig1}}
\end{figure}

\begin{table*}
\centering
\caption{Summary of the Experimental Results for Trucks Data\label{tab3}}
(An entry $w/l$ means THORS wins $w$ times and lose $l$ times)
\begin{tabular}{|c|c|c|c|c|c|}
\hline
\diagbox{Base Classifier}{Algorithm} &Theoretical &Null &Empirical &Metacost &CRS\\
\hline
Logit &20/0	&20/0	&15/5	&20/0	&20/0\\
\hline
DT &20/0	&20/0	&10/10	&20/0	&20/0\\
\hline
NB &20/0	&20/0	&20/0	&8/12	&20/0\\
\hline
LDA &20/0	&20/0	&20/0	&20/0 &20/0\\
\hline
\end{tabular}
\end{table*}

We list average costs and their deviance for each algorithm on each classifier in Table \ref{tab2}, from which we can see that for Logit, DT, and LDA classifier, THORS reaches the least average cost with small deviance in various methods. Figure \ref{fig1} exhibited average costs and their deviance for each algorithm on each classifier, from which we can see that for Logit, DT, and LDA classifier, THORS reaches the least average cost with small deviance in various methods. Figure \ref{fig2} presents the box-plots of cost for each approach on different classifiers. It's obvious to see that the box-plot of THORS is always at the bottom, which means that THORS is always among the best algorithms for various classifiers. And Table \ref{tab3} reports the detailed comparison results. THORS wins at least half of 20 rounds in almost all cases (except for Metacost method on NB classifier), beating other strategies in these cases.

\begin{table*}
\centering
\caption{Average Running Time of Some Methods for Trucks Data (unit: s)\label{tab4}}
\begin{tabular}{|c|c|c|c|c|}
\hline
\diagbox{Base Algorithm}{Thresholding} &THORS &Empirical &Metacost &CRS\\
\hline
Logit &1.60	&17.61 &6.92 &0.27	 \\
\hline
DT &2.93 &45.60 &33.06 &0.94 \\
\hline
NB &1.33  &11.77 &0.71 	&0.07  	\\
\hline
LDA &1.60 &13.25  &1.98	&0.10 \\
\hline
\end{tabular}
\end{table*}

\begin{figure}[htb]
\centering
\includegraphics[width=0.45\textwidth]{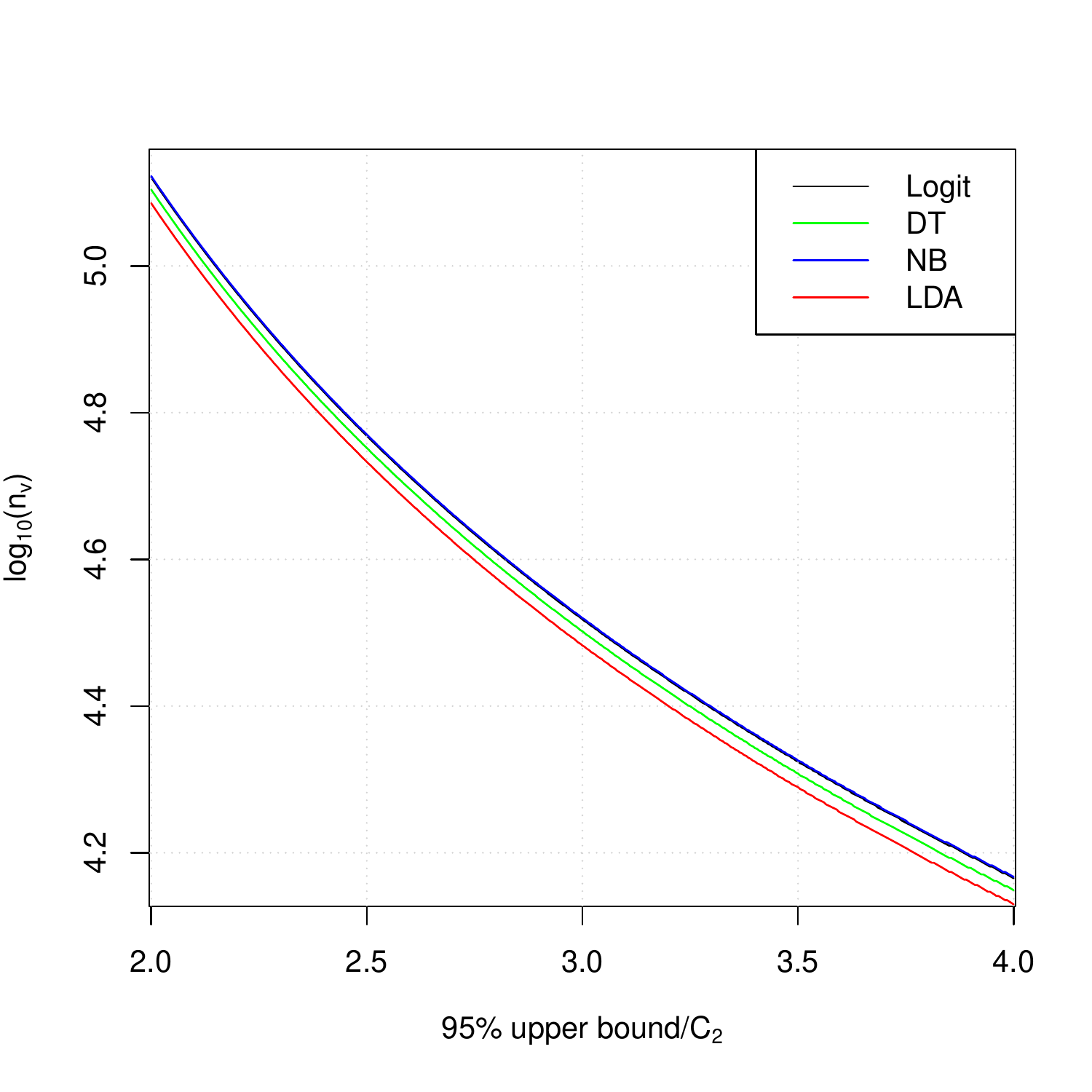}
\caption{Size of Validation Set and Upper Bound of Cost for Trucks Data\label{fig2}}
\end{figure}

Average running time for one round of each algorithm is listed in Table \ref{tab4}. And relationship between estimated minimal size of validation set (after logarithmic transform), $n_v$, and 95\% upper bound are exhibited in Figure \ref{fig2}. From Table \ref{tab4}, it is noticed that THORS is always more time-economic in the comparison with Metacost and empirical method. Although CRS always takes the least time due to its simple re-sampling procedure, the performance of it is always much worse than that of THORS. Finally from Figure \ref{fig2},  it's available to control 95\% upper bound under $3.5C_2$ for four classifiers using present 24,000 instances in validation set by THORS. We can see that four curves are approximately straight lines, showing an exponential relation between the upper bound and validation set size. And it's a tradeoff for the increasing sample size and a tighter upper bound. There is no obvious difference for the requirements of validation set size under the same ratio between the upper bound and $C_2$ between four classifiers for THORS algorithm.

\subsection{Income Data}
The second dataset we choose is an adult dataset containing 32,561 income observations from UCI Machine Learning Repository. Samples belong to two classes: 24,720 people whose income is below 50,000 (class 0, negative) and 7,841 ones whose income is over 50,000 (class 1, positive). The imbalance rate here is 3.15:1, much smaller than the first dataset. We also choose 10 predictors as predictors from 14 other characteristics of each person applying ANOVA F-value, for training base classifiers. To create a cost-sensitive problem we let cost for \textit{FN} and \textit{FP} be 100 and 10, respectively. Among these 32,561 observations, 13,025 instances are used to train base models, 13,024 are utilized for thresholding, and remained 6,512 samples are for evaluating the performance of these approaches. Base classifiers we choose are logistic regression with cost weighting (Logit), Na\"ive Bayes (NB), linear discriminant analysis (LDA), and random forest (RF). We also make a comparison of the performance of THORS with other thresholding schemes such as empirical thresholding method, theoretical thresholding, and other meta-learning strategies, including Metacost and CRS. And each algorithm is also run for 20 times.

\begin{table*}
\centering
\caption{Average Costs and Standard Deviations for Each Algorithm on Each Classifier for Income Data\label{tab5}}
\begin{tabular}{|c|c|c|c|c|c|c|}
\hline
&THORS &Null &Theoretical &Empirical &Metacost &CRS\\
\hline
Logit &$35469 \pm 1371$	&$35584 \pm 1066$ &$49144 \pm 400$ &$46176 \pm 2593$ &$36817 \pm 1117$ &$44717 \pm 3190$
 \\
\hline
LDA &$35689 \pm 934$ &$95099 \pm 2497$ &$36252 \pm 967$ &$45749 \pm 3417$ &$39382 \pm 1009$ &$39492 \pm 1603$
 \\
\hline
NB &$32901 \pm 934$ &$108761 \pm 2883$ &$67100 \pm 3360$ &$49455 \pm 303$ &$110869 \pm 2667$ &$88787 \pm 5488$
 \\
\hline
RF &$29419 \pm 779$ &$57135 \pm 2131$ &$30415 \pm 1338$ &$29507 \pm 1172$ &$29769 \pm 1497$ &$32378 \pm 1185$
\\
\hline
\end{tabular}
\end{table*}

\begin{figure}[htb]
\centering
  \begin{tabular}{@{}cccc@{}}
    \includegraphics[width=.450\textwidth]{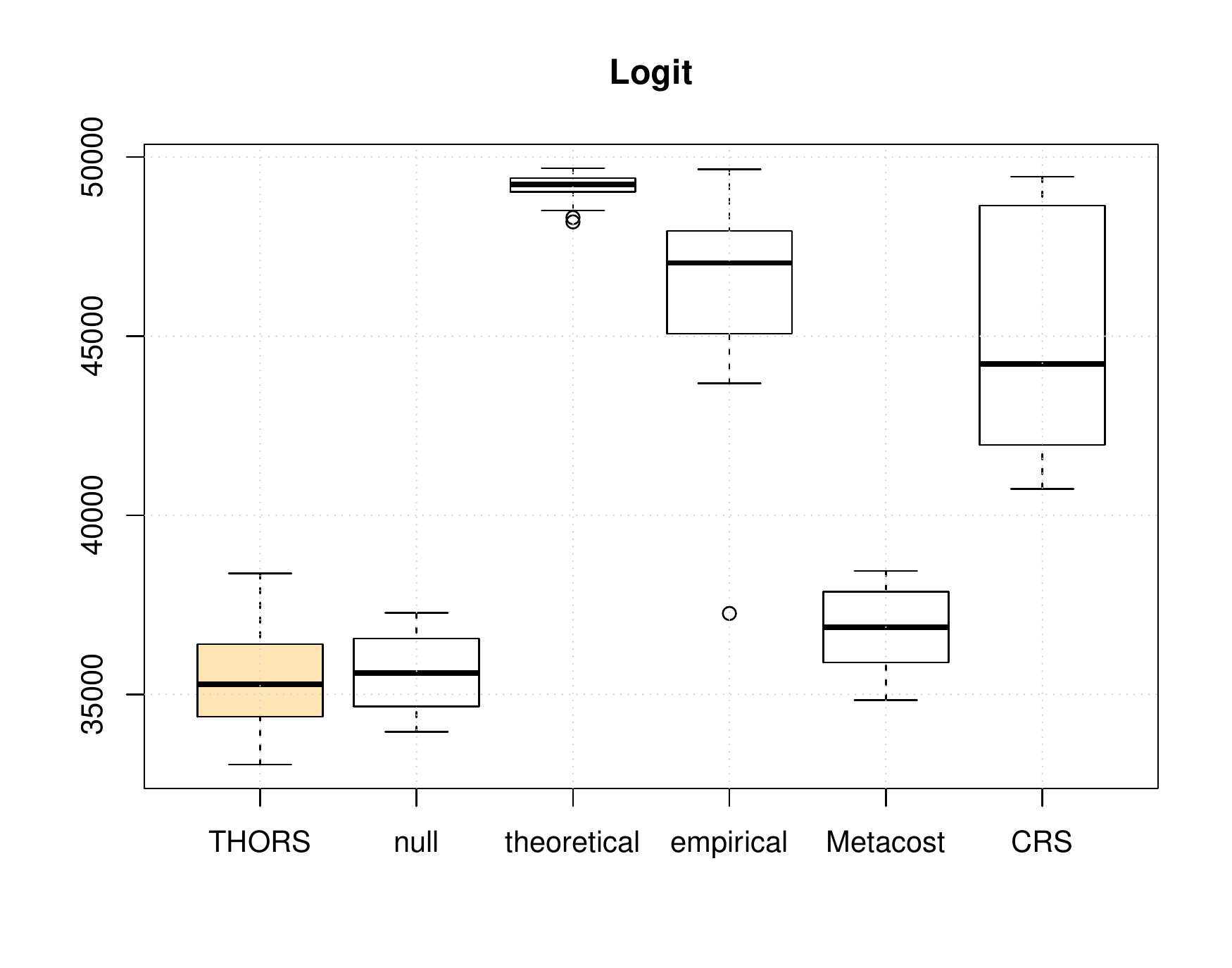} &
    \includegraphics[width=.450\textwidth]{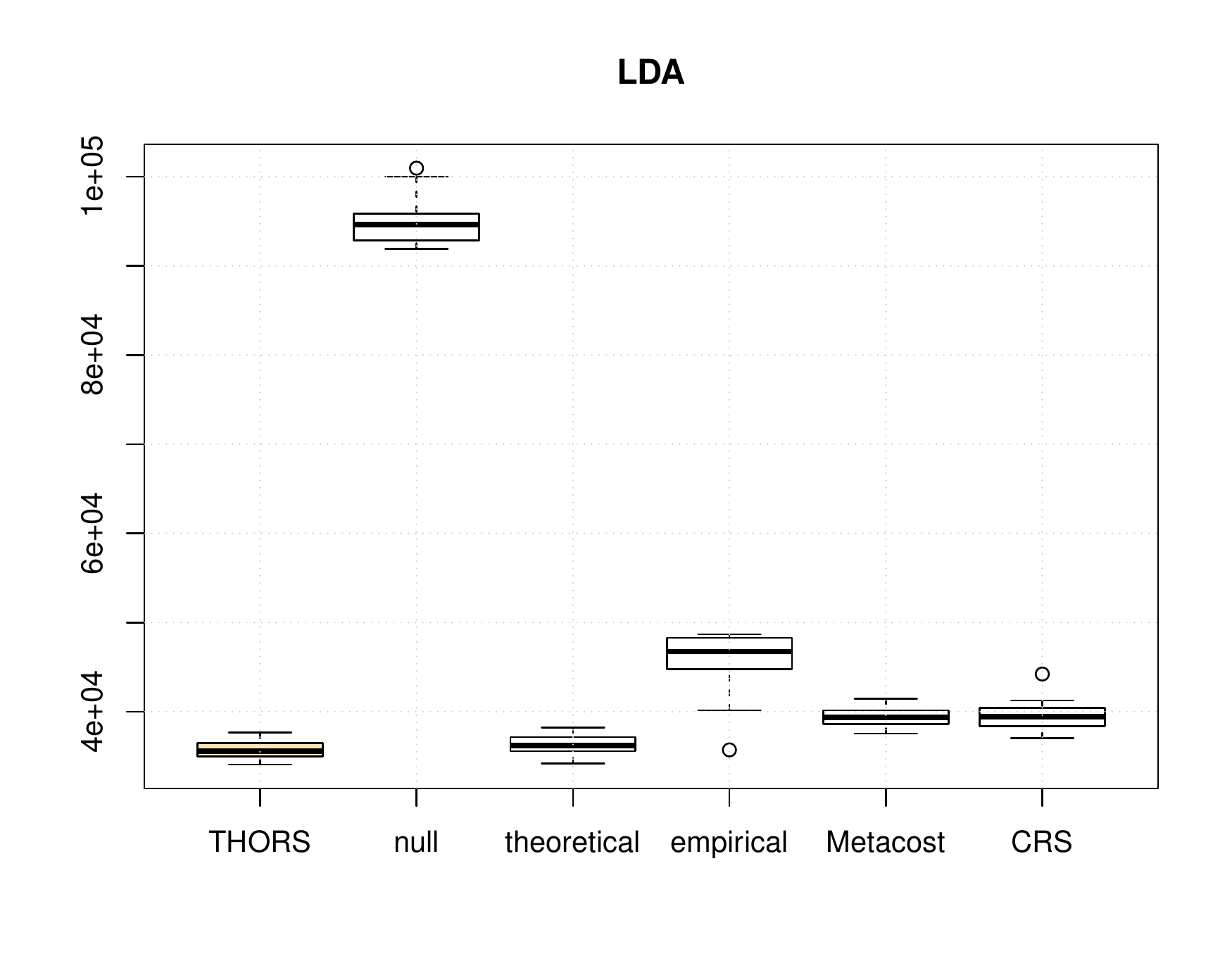}   \\
    \includegraphics[width=.450\textwidth]{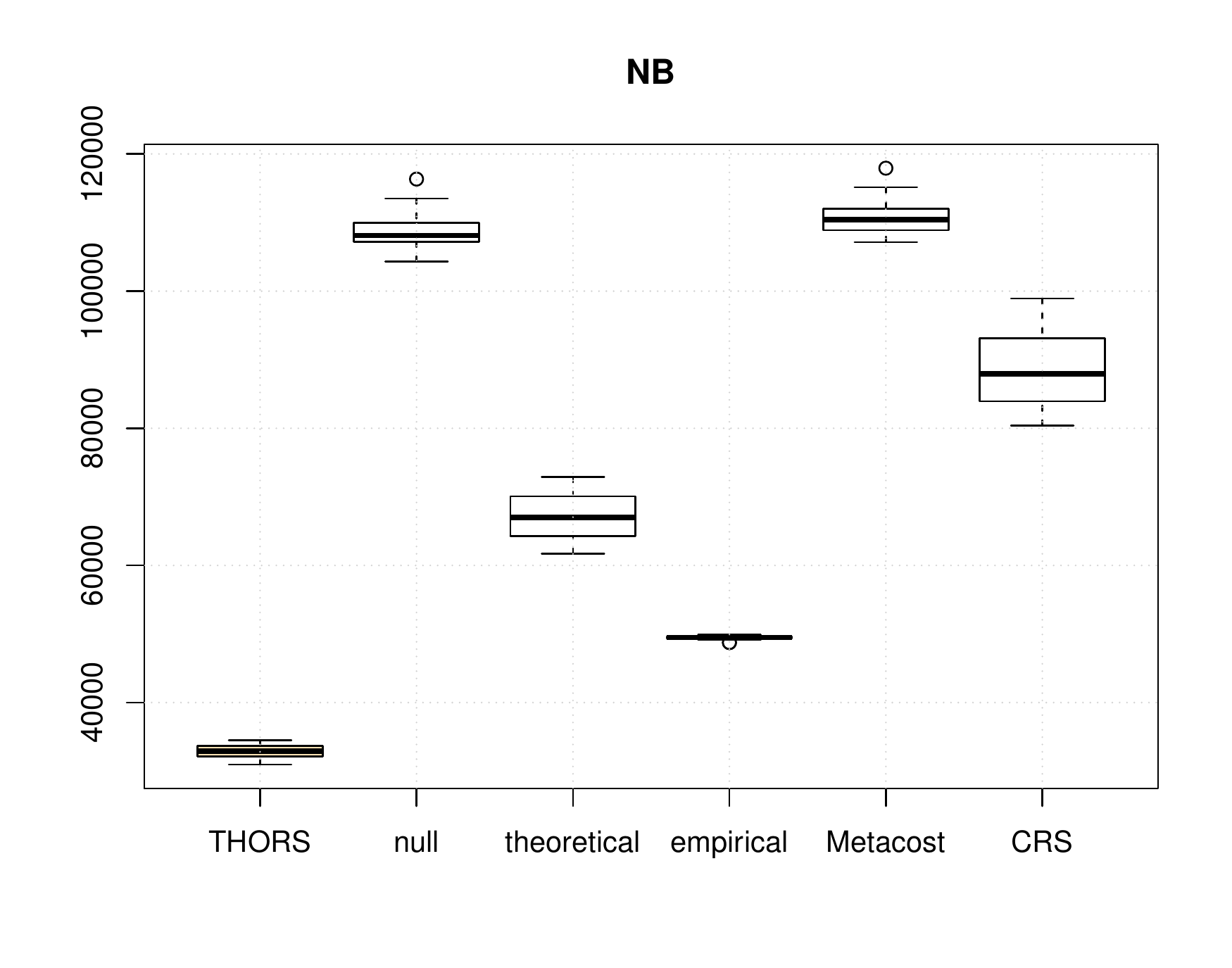} &
    \includegraphics[width=.450\textwidth]{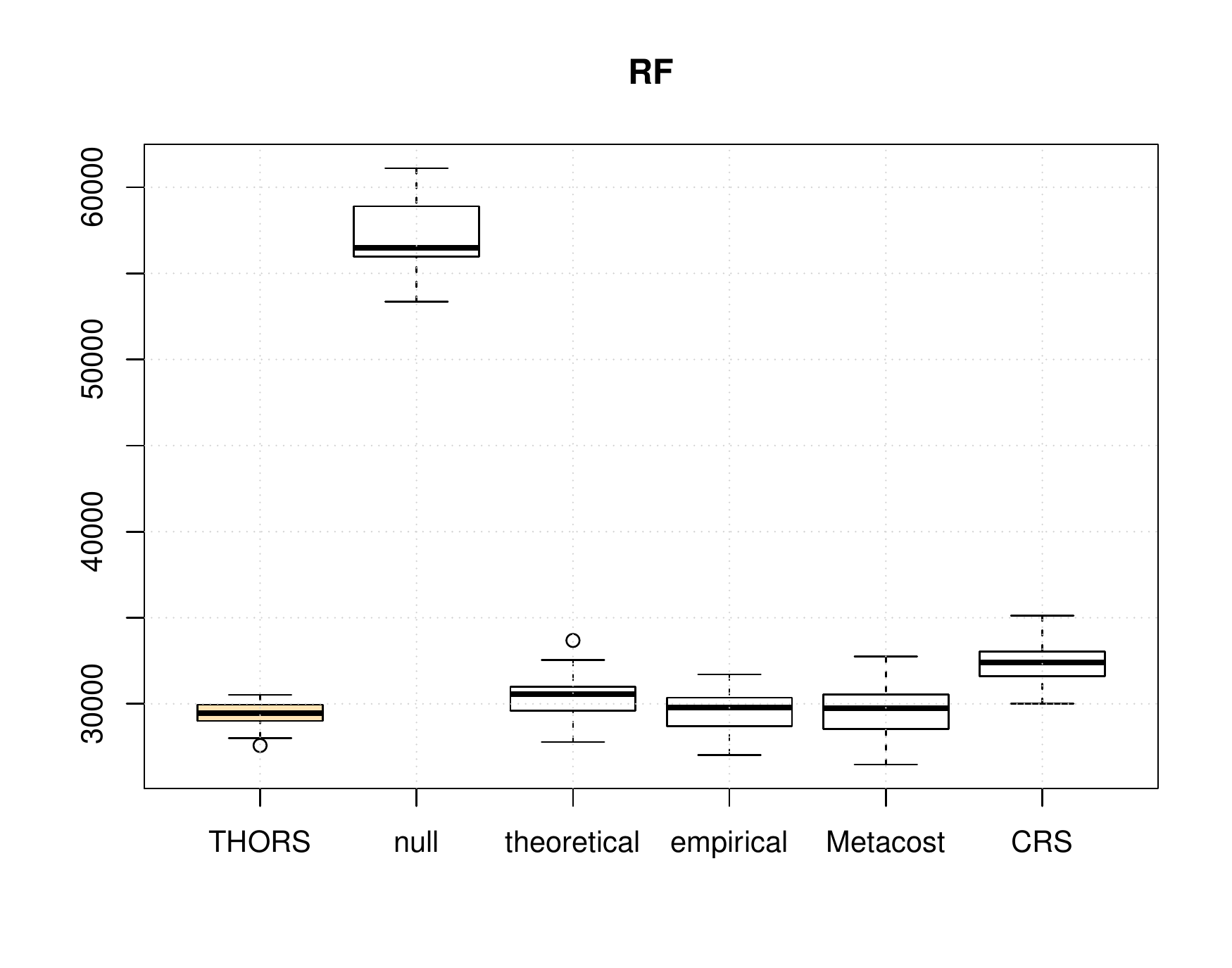}   \\
  \end{tabular}
  \caption{Box-plots of Costs by Different Thresholding Methods\label{fig3}}
\end{figure}

\begin{table*}
\centering
\caption{Summary of the Experimental Results for Income Data}
(An entry w/l means our approach win w times and lose l times\label{tab6})
\begin{tabular}{|c|c|c|c|c|c|}
\hline
\diagbox{Base Algorithm}{Thresholding} &Theoretical &Null &Empirical &Metacost &CRS\\
\hline
Logit &20/0	&12/8	&20/0	&18/2	&20/0\\
\hline
LDA &18/2	&20/0	&20/0	&20/0	&20/0\\
\hline
NB &20/0	&20/0	&20/0	&20/0	&20/0\\
\hline
RF &17/3	&20/0	&13/7	&11/9  &20/0\\
\hline
\end{tabular}
\end{table*}

Table \ref{tab5} shows the average costs and corresponding deviance on the test set of various thresholding and meta-learning approaches. And win/loss comparison results are summarized in Table \ref{tab6}, from which we can notice that THORS can always beat other methods in more than half of 20 rounds. In addition, THORS even wins all the 20 rounds for NB base classifier in the comparison with all other algorithms. We can see that for all the four classifiers, THORS always gets the minimal average cost with small deviance, showing the power of THORS. Box-plots for different approaches are shown in Figure \ref{fig3}, from which it can be observed that box-plot of THORS always stays near the bottom, representing a low cost on test set.

\begin{table*}
\centering
\caption{Average Running Time of Some Methods (unit: s)\label{tab7}}
\begin{tabular}{|c|c|c|c|c|}
\hline
\diagbox{Base Algorithm}{Thresholding} &THORS &Empirical &Metacost &CRS\\
\hline
Logit &1.01	&15.88 &41.51 &0.66	 \\
\hline
LDA &0.98 &13.11 &27.88 &0.46 \\
\hline  	 	 	
NB &0.83 &10.94 &0.43 	&0.05 	\\
\hline
RF &1.32 &54.39  &13.84  &0.53 \\
\hline
\end{tabular}
\end{table*}

\begin{figure}[htb]
\centering
\includegraphics[width=0.45\textwidth]{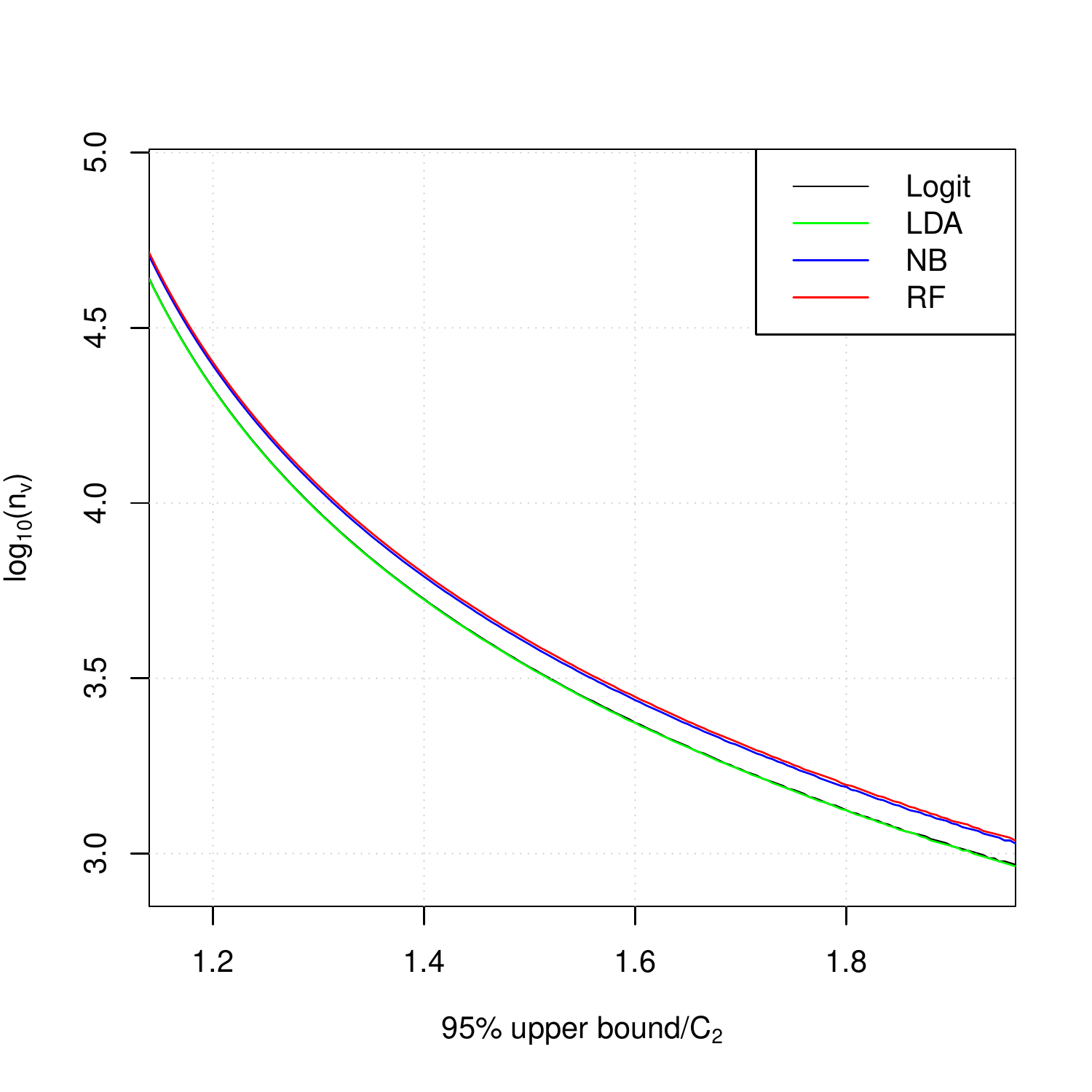}
\caption{Minimal size of validation Set under 95\% specific upper bound\label{fig4}}
\end{figure}

Table \ref{tab7} exhibits the average running time for a single round for various cost-sensitive algorithms, showing us that THORS is very efficient. And Figure \ref{fig4} presents the relation between 95\% upper bound and the conservative estimation of minimal size of validation set. We can see the similar approximately linear relation between the logarithmic size and the upper bound as in the previous case. It's easy to control the expected costs under 1.3$C_2$ in the probability of 95\% for the present validation set size. And the size will also boom when the upper bound decreases. This result also indicates the trade-off we mentioned in the case of the trucks data set.

\subsection{Telescope Data}
In this case, we investigate the MAGIC gamma telescope data, which is from UCI Machine Learning Repository. All the 19,020 instances are divided into two classes, including 12,332 class 0 samples and 6,688 class 1 ones. The imbalance rate is 1.84:1. Except for classes, there are 10 other attributes, which will be set as predictors in the models. To formulate this problem into a cost-sensitive one, we set the costs as 100 and 20 for False Negative case and False Positive case. Among 19,020 instances, 7,608 observations are used for training base classifier, denoted as training set, 7,608 ones are applied to choose the optimal threshold, denoted as validation set, and remained 3,804 observations are set as the test set. And the base classifiers we choose include logistic regression with cost weighting (Logit), linear discriminant analysis (LDA), Na\"ive Bayes (NB), and random forest (RF). The same as previous two datasets, we will compare results of THORS with other thresholding methods, including the null model, theoretical method and empirical method, and other meta-learning approaches, including Metacost and CRS. Each algorithm is also run for 20 rounds.

\begin{table*}
\centering
\caption{Average Costs and Standard Deviations for Each Algorithm on Each Classifier for Telescope Data\label{tab8}}
\begin{tabular}{|c|c|c|c|c|c|c|}
\hline
&THORS &Null &Theoretical &Empirical &Metacost &CRS\\
\hline
Logit &$21773 \pm 544$	&$24713 \pm 241$ &$23534 \pm 592$ &$21606 \pm 586$ &$24565 \pm 514$ &$24713 \pm 241$
 \\
\hline
LDA &$23924 \pm 713$ &$60858 \pm 2028$ &$23833 \pm 600$ &$24500 \pm 351$ &$25301 \pm 595$ &$24334 \pm 910$
 \\
\hline
NB &$24613 \pm 532$ &$84904 \pm 2550$ &$65014 \pm 2675$ &$24713 \pm 241$ &$82121 \pm 2291$ &$73854 \pm 3756$
 \\
\hline
RF &$16948 \pm 596$ &$32836 \pm 2131$ &$16633 \pm 394$ &$20142 \pm 1080$ &$17735 \pm 802$ &$18313 \pm 982$
\\
\hline
\end{tabular}
\end{table*}

\begin{figure}[htb]
\centering
  \begin{tabular}{@{}cccc@{}}
    \includegraphics[width=.450\textwidth]{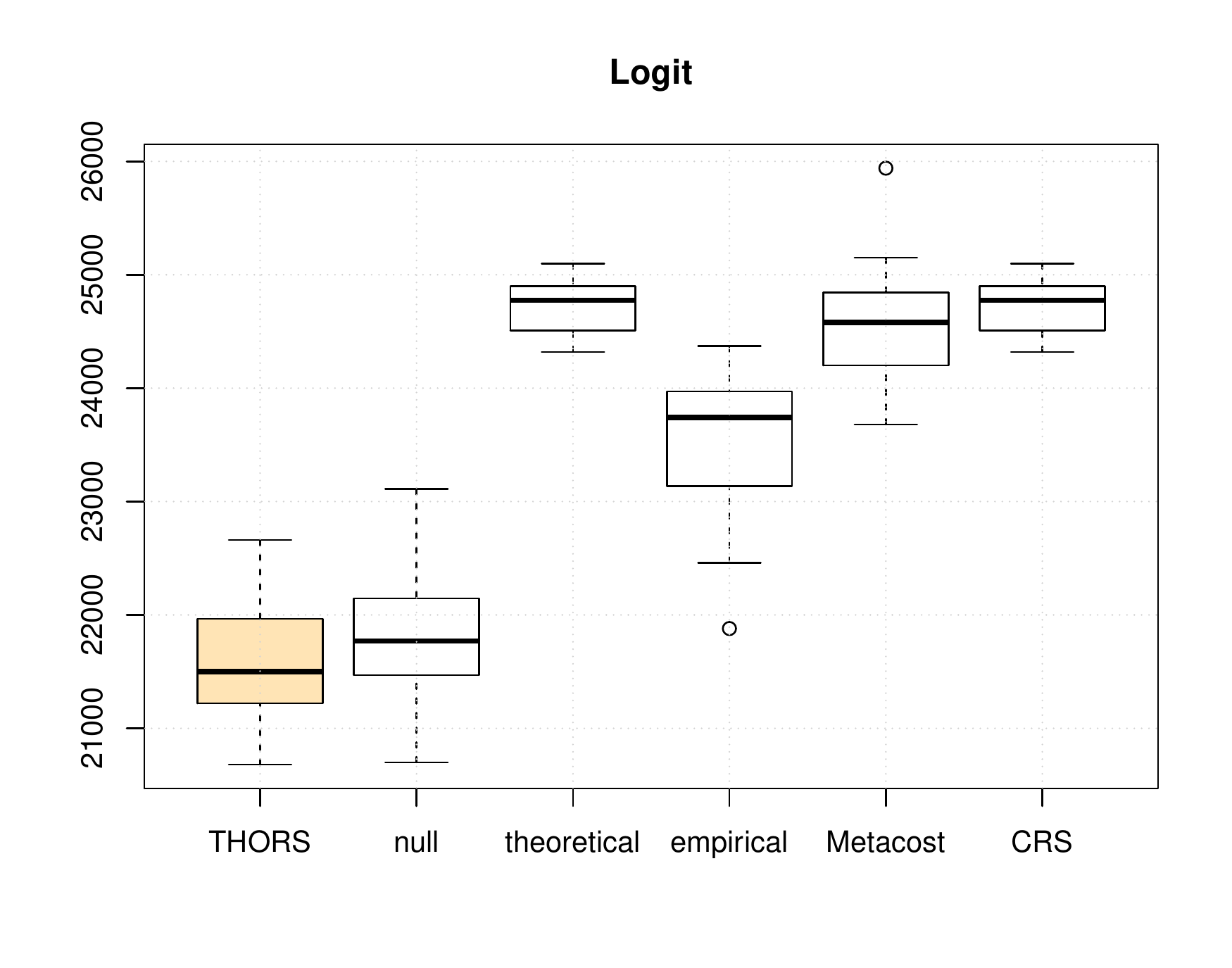} &
    \includegraphics[width=.450\textwidth]{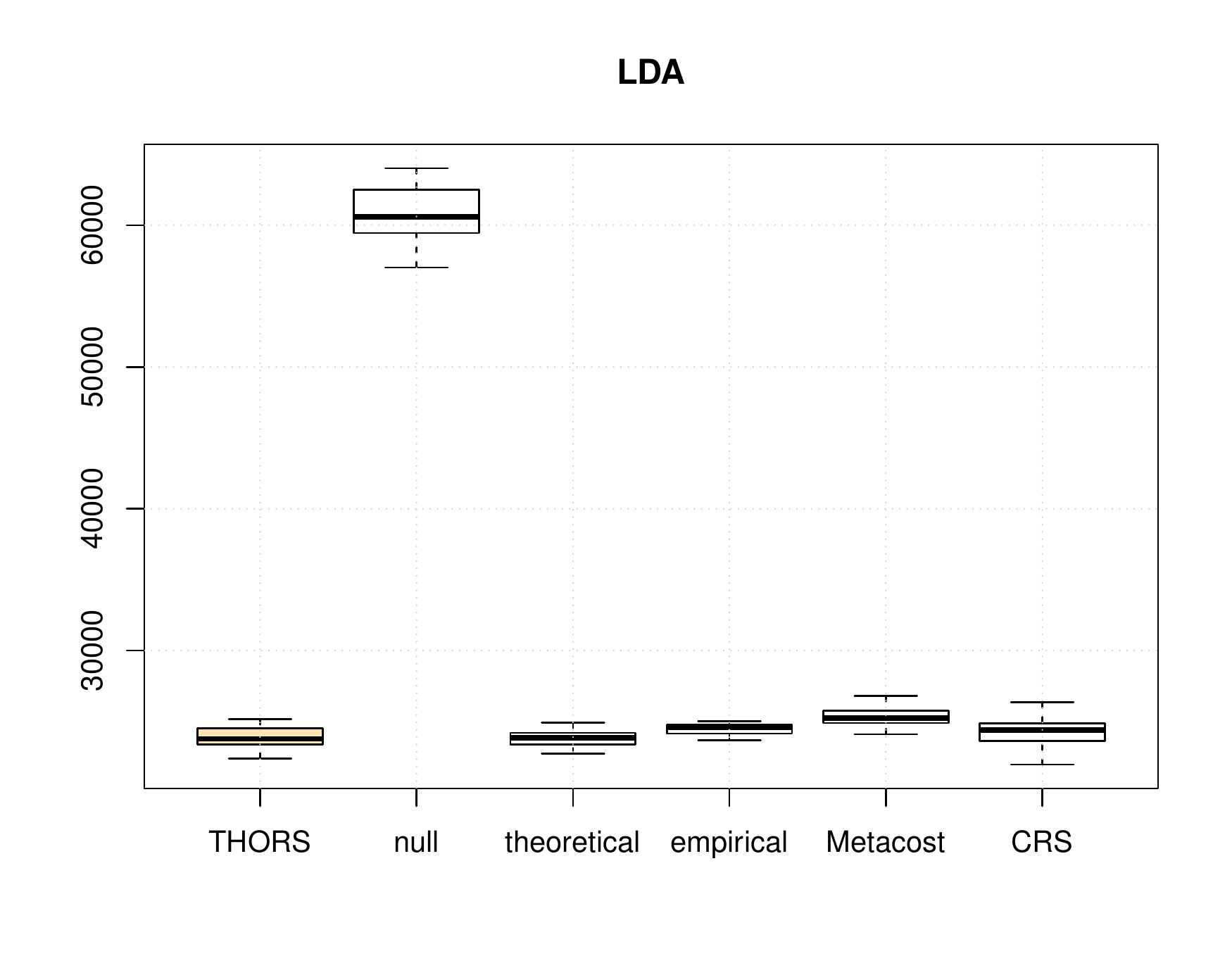}   \\
    \includegraphics[width=.450\textwidth]{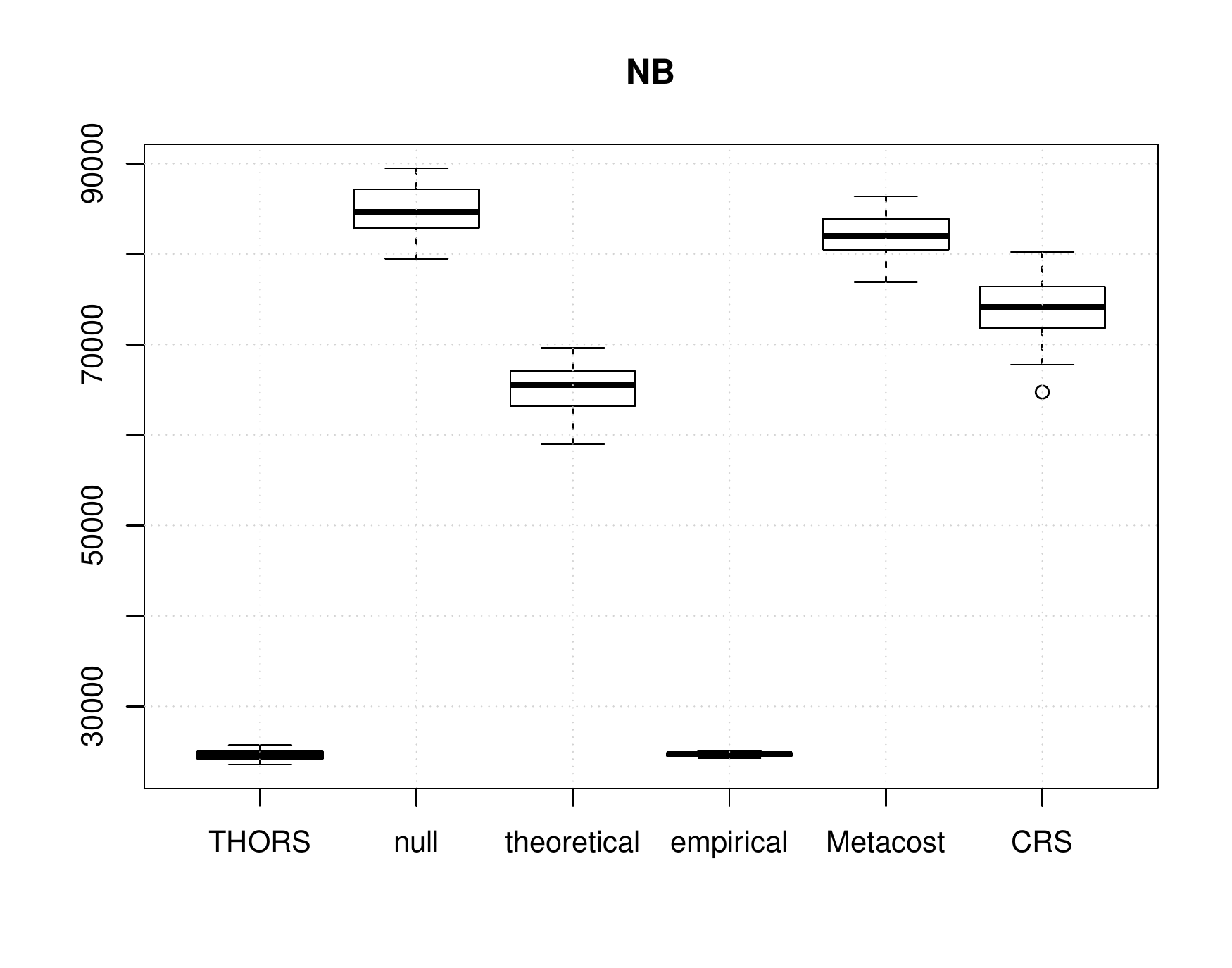} &
    \includegraphics[width=.450\textwidth]{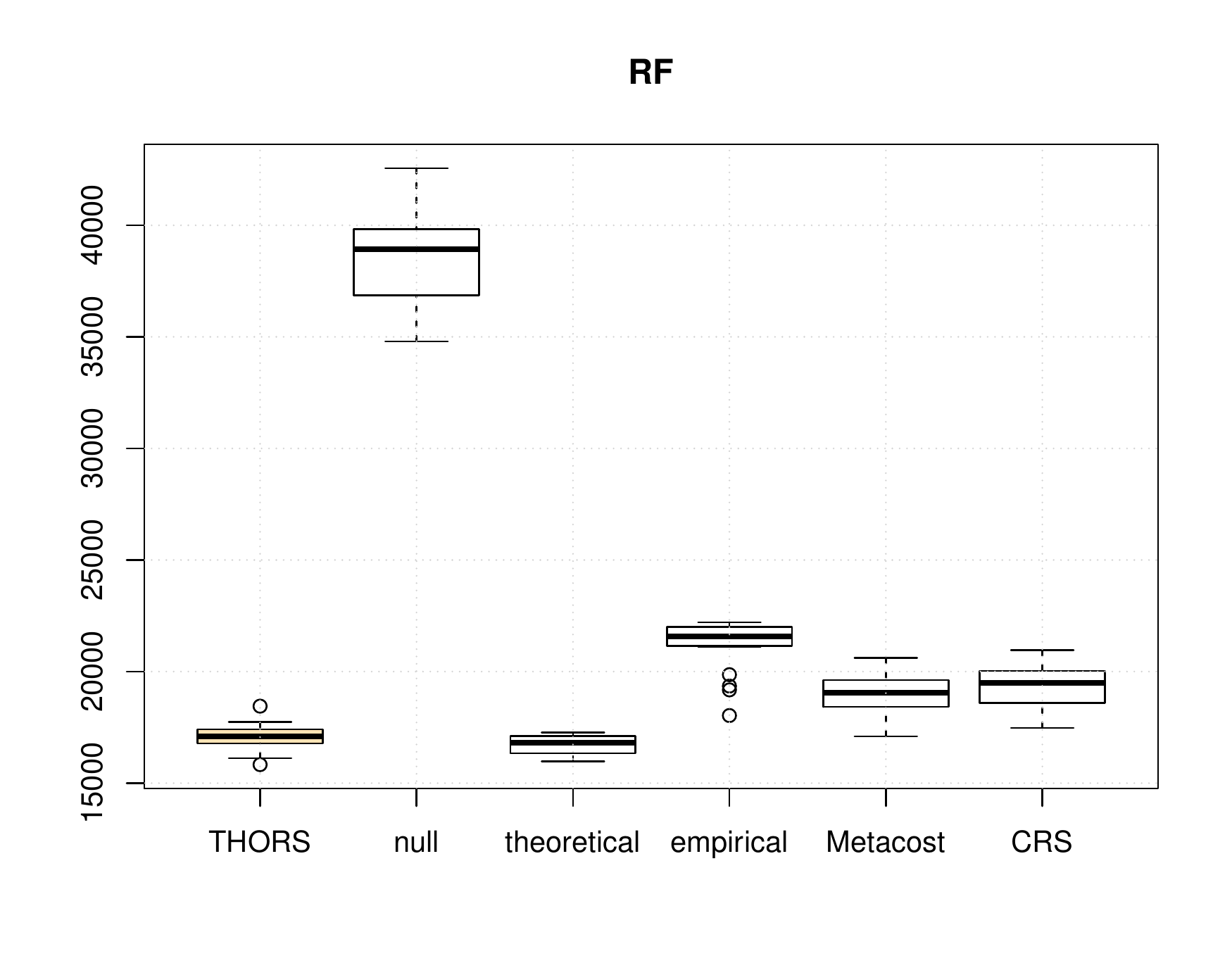}   \\
  \end{tabular}
  \caption{Box-plots of Costs by Different Thresholding Methods for Telescope Data\label{fig5}}
\end{figure}

\begin{table*}
\centering
\caption{Summary of the Experimental Results for Telescope Data}
(An entry $w/l$ means our approach win $w$ times and lose $l$ times\label{tab9})
\begin{tabular}{|c|c|c|c|c|c|}
\hline
\diagbox{Base Algorithm}{Thresholding} &Theoretical &Null &Empirical &Metacost &CRS\\
\hline
Logit &20/0	&13/7	&20/0	&20/0	&20/0\\
\hline
LDA &9/11	&20/0	&16/4	&19/1	&16/4\\
\hline
NB &20/0	&20/0	&13/7	&20/0	&20/0\\
\hline
RF &7/13	&20/0	&20/0	&20/0	&19/1\\
\hline
\end{tabular}
\end{table*}

It can be noticed from Table \ref{tab9} that for almost all the cases, THORS defeats other methods over half of 20 rounds (except for theoretical method for LDA and RF classifier). And from Table \ref{tab8}, we can observe that the difference between average costs of THORS and theoretical method for LDA and RF classifier is very small. We can also see from Figure \ref{fig5} that box-plot of THORS always stay at the bottom of the figure.

\begin{table*}
\centering
\caption{Average Running Time of Some Methods (unit: s)\label{tab10}}
\begin{tabular}{|c|c|c|c|c|}
\hline
\diagbox{Base Algorithm}{Thresholding} &THORS &Empirical &Metacost &CRS\\
\hline
Logit &0.55	&11.45 &14.74 &0.56	 \\
\hline
LDA &0.50 &9.51 &9.93 &0.39 \\
\hline  	 	 	
NB &0.48 &9.10 &7.52  &0.30 	\\
\hline
RF &0.77 &37.98  &24.85  &0.34 \\
\hline
\end{tabular}
\end{table*}

\begin{figure}[htb]
\centering
\includegraphics[width=0.45\textwidth]{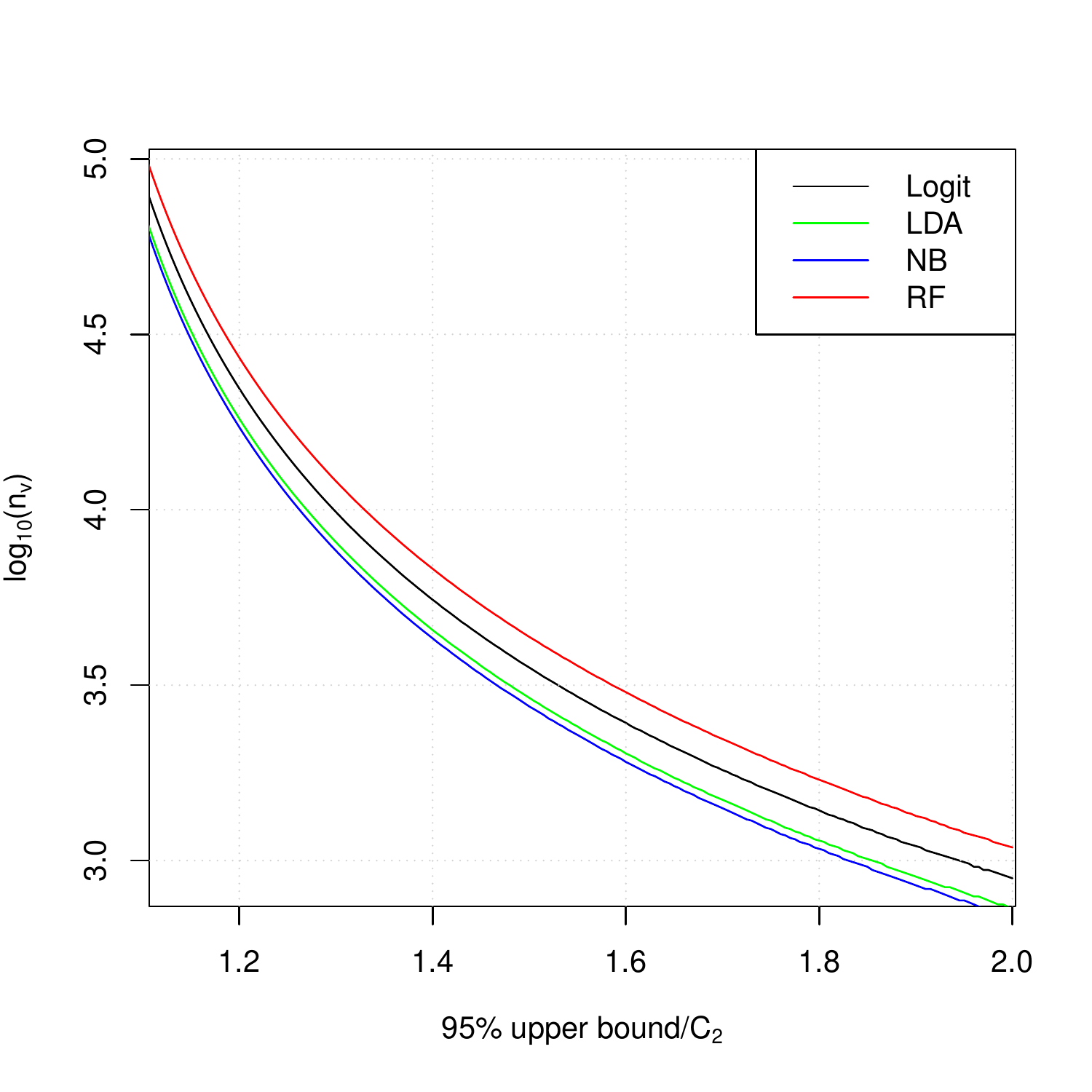}
\caption{Minimal size of validation Set under 95\% specific upper bound\label{fig6}}
\end{figure}

Table \ref{tab10} lists all the average running times for each algorithm, and from this we can observe that THORS is very efficient. And Figure \ref{fig6} shows us that it is possible to control the expected cost under about $1.3C_2$ in 95\% probability for the present validation set size. Under the same ratio between upper bound and $C_2$, RF classifier seems to need a little more samples in the validation set than other three base classifiers.
\section{Discussion}
We propose an effective and efficient thresholding algorithm THORS to make an arbitrary scoring-type classifier whose cumulative distribution function of scores is continuous into a cost-sensitive one. Its idea of using order statistic to classify is intuitive and simple. THORS usually results in an excellent performance in terms of cost and time complexity. It can always induce the lowest cost in a short time among various algorithms. Besides, different from other popular cost-sensitive methods, we prove that THORS has several theoretical properties including the bound of the expected cost and an asymptotic boundary, which can also be used to estimate the size of validation to find optimal threshold controlling the expected cost with a specific probability. Finally, THORS often achieves drastic savings in computational resources for its low time complexity, which is desirable for applications that involve a massive amount of data.

One direction for the future work is the application of THORS to multiclass cost-sensitive problems. Here we introduce a simple idea for extending THORS to multiclass cases. Supposed that we have $m$ classes noted as class $1,2,...,m$. And the costs of misclassifying them are decreasing correspondingly. Then we will find an optimal threshold vector with length $m-1$.  Here the full data is also divided into the training set and validation set to train base models and find optimal thresholds, respectively. Firstly we can sort the misclassification cost of different classes by decreasing order $s_1,...,s_m$. Next, we apply THORS on the validation set in terms of class $s_1$ to get a threshold, $c_1$. Estimated expected cost on validation set can be similarly defined as (\ref{escost}). Then optimal order statistic of one score corresponding to minimal empirical expected cost will be picked up as the first threshold. Next, we can apply THORS again on remained unlabeled observations in validation set and get the second threshold $c_2$. Similarly, all $m-1$ thresholds $c_1,c_2,...,c_{m-1}$ will be determined. With the threshold vector $c=(c_1,c_2,...,c_{m-1})$, we can construct the following cost-sensitive classifier: Given a new data point, we firstly decide whether it should be classified into class $s_1$ or not on the basis of $c_1$. If it does not fall into class $s_1$, then we calculate the score estimating its possibility belonging to class $s_2$ and assign it to class $s_2$ if the score is lower than $c_2$. The process holds on until the instance is labeled. Another interesting problem is the case that the cumulative distribution function of the score is not continuous. There may be other ways to yield comparable conclusions in discontinuity case. Besides, THORS can be combined with other approaches like resampling, which may lead to a better performance for imbalanced problem.

\section*{Acknowledgments}
This work is supported by the National Key Research and Development Plan (No. 2016YFC0800100) and the NSF of China (No. 11671374, 71771203).

\bibliographystyle{ims}
\bibliography{references}

\section*{Appendix: Proofs}
\setcounter{equation}{0}
\renewcommand{\theequation}{A.\arabic{equation}}

\subsection*{Proof of Theorem 1}
\begin{proof}
For (\ref{classification function}) and the relation between $k_0, k_1$ and $k$ illustrated in Theorem \ref{th1}, class 0 observations scored as $T^0_{(k_0+1)},T^0_{(k_0+2)},...,T^0_{(n_0)}$ will be classified into class 1 and class 1 observations scored as $T^1_{(1)},T^1_{(2)},...,T^1_{(k_1)}$ will be classified into class 0 if we choose $T_{(k)}$ as the threshold. Then empirical \textit{False Negative Rate (FNR)} denoted as $\widehat{\alpha}_1$ and empirical \textit{False Positive Rate (FPR)} denoted as $\widehat{\alpha}_0$ on validation set satisfy
\begin{equation}
\widehat{\alpha}_{i0} = \frac{n_{0}-k_{0}}{n_{0}},
\end{equation}
\begin{equation}
\widehat{\alpha}_{i1} = \frac{k_{1}}{n_{1}}.
\end{equation}
And for each instance, the cost can be expressed as
\begin{equation}
C(j,k) =
\left\{
\begin{array}{rcl}
\beta A      &,      &j=0\,\,and\,\,k=1 \\
A     &,      &j=1\,\,and\,\,k=0 \\
0     &,      &otherwise.\\
\end{array}
\right.
\end{equation}
And the marginal probabilities of class 0 and 1, are estimated as $\widehat{\pi}_0=\frac{n_0}{n_v}$ and $\widehat{\pi}_1=\frac{n_1}{n_v}$. Thus, by (\ref{escost}), the expected cost on validation set satisfies
\begin{equation}
\begin{aligned}
\widehat{C} &= \sum_{i=1}^{n_v}[C(0,1)\widehat{\alpha}_{i0}\widehat{\pi}_0+C(1,0)\widehat{\alpha}_{i1}\widehat{\pi}_1]
= \sum_{i=1}^{n_v}(\beta A\widehat{\pi}_0\widehat{\alpha}_{i0}+A\widehat{\pi}_1\widehat{\alpha}_{i1}) \\
&= n_vA\Big(\beta\widehat{\pi}_0+\frac{\widehat{\pi}_1k_1}{n_1}-\frac{\beta\widehat{\pi}_0k_0}{n_0}\Big)
\end{aligned}
\end{equation}
which means that we only need to minimize $\frac{\widehat{\pi}_1 k_1}{n_1} - \frac{\beta \widehat{\pi}_0 k_0}{n_0}$ to minimize $\widehat{C}$. This completes the proof.
\end{proof}

\subsection*{Proof of Theorem 2}
\begin{proof}
In the following notation ``$\cong$'' represents ``defined as''. By definitions of $k_0$ and $k_1$ in Theorem \ref{th1}, we know that threshold $c = T_{(k)}$ satisfies
\begin{equation}
max\{T^0_{(k_0)},T^1_{(k_1)}\} \leq T_{(k)} \leq min\{T^0_{(k_0+1)},T^1_{(k_1+1)}\},
\end{equation}
indicating that \textit{FPR} of instance $i$
\begin{equation}
\begin{aligned}
\alpha_{i0} &= P(T_i>T_{(k)}|class(i)=0)\leq P(T_i>T^0_{(k_0)}|class(i)=0)\\
&\cong P(T^0_i>T^0_{(k_0)}) = y_{i0},
\end{aligned}
\end{equation}
and also
\begin{equation}
\begin{aligned}
\alpha_{i0} &= P(T_i>T_{(k)}|class(i)=0)\geq P(T_i>T^0_{(k_0+1)}|class(i)=0)\\
&\cong P(T^0_i>T^0_{(k_0+1)}) = z_{i0}.
\end{aligned}
\end{equation}
Similarly, for \textit{FNR} denoted as $\alpha_{i1}$, we have
\begin{equation}
\begin{aligned}
\alpha_{i1} &= P(T_i \leq T_{(k)}|class(i)=1)\leq P(T_i \leq T^1_{(k_1+1)}|class(i)=1)\\
&\cong P(T^1_i \leq T^1_{(k_1+1)}) = y_{i1},
\end{aligned}
\end{equation}

\begin{equation}
\begin{aligned}
\alpha_{i1} &= P(T_i \leq T_{(k)}|class(i)=1)\geq P(T_i \leq T^1_{(k_1)}|class(i)=1)\\
&\cong P(T^1_i \leq T^1_{(k_1)}) = z_{i1}.
\end{aligned}
\end{equation}
Above all, two types of error can be controlled as
\begin{align}
&z_{i0} \cong P(T^0_i > T^0_{(k_0+1)}) \leq\alpha_{i0} \leq P(T^0_i > T^0_{(k_0)}) \cong y_{i0},\label{FPRbd} \\
&z_{i1} \cong P(T^1_i \leq T^1_{(k_1)}) \leq\alpha_{i1} \leq P(T^1_i \leq T^1_{(k_1+1)}) \cong y_{i1}. \label{FNRbd}
\end{align}
For the randomness of order statistic, here $z_{i0}, z_{i1}, y_{i0}, y_{i1}$ are all random variables instead of constants. Now let's investigate distributions of $y_{i1}, y_{i0}, z_{i1}, z_{i0}$. Firstly we denote the conditional cumulative distribution function $P(T_i \leq x|class(i)=0) = P(T^0_i \leq x)$ as $F_0(x)$ and $P(T_i \leq x|class(i)=1) = P(T^1_i \leq x)$ as $F_1(x)$. Hence,
\begin{equation}
\begin{aligned}
P(y_{i0}\leq x) &= P(P(T^0_i > T^0_{(k_0)}) \leq x)= P(1 - F_0(T^0_{(k_0)}) \leq x)\\
&= P(F_0(T^0_{(k_0)}) \geq 1 - x)= P(T^0_{(k_0)} \geq F_0^{-1}(1-x))\\
&= P(at\, least\, (n_0-k_0+1)\,\,of\,\,T^{0}_{i}\,'s\,\,are\,\,no\,\,less\,\,than\,\, F_0^{-1}(1-x))\\
&= \sum_{j=n_0-k_0+1}^{n_0}\binom{n_0}{j}[1-F_0(F_0^{-1}(1-x))]^{j}[F_0(F_0^{-1}(1-x))]^{n_0-j}\\
&= \sum_{j=n_0-k_0+1}^{n_0}\binom{n_0}{j}x^j(1-x)^{n_0-j},
\end{aligned}
\end{equation}
\begin{equation}
\begin{aligned}
P(y_{i1}\leq x) &= P(P(T^1_i \leq T^1_{(k_1+1)}) \leq x)= P(F_1(T^1_{(k_1+1)}) \leq x)\\
&= P(T^1_{(k_1+1)} \leq F_1^{-1}(x))= P(at\,\,least\,\,(k_1+1)\,\,of\,\,T^{1}_{i}\,'s\,\,are\,\,no\,\,more\,\,than\,\, F_1^{-1}(x))\\
&= \sum_{j=k_1+1}^{n_1}\binom{n_1}{j}[F_1(F_1^{-1}(x))]^{j}[1-F_1(F_1^{-1}(x))]^{n_1-j}\\
&= \sum_{j=k_1+1}^{n_1}\binom{n_1}{j}x^j(1-x)^{n_1-j}.
\end{aligned}
\end{equation}
Similarly there hold
\begin{equation}
P(z_{i0} \leq x) = \sum_{j=n_0-k_0}^{n_0}\binom{n_{0}}{j}x^{j}(1-x)^{n_{0}-j},
\end{equation}
\begin{equation}
P(z_{i1} \leq x) = \sum_{j=k_{1}}^{n_{1}}\binom{n_{1}}{j}x^{j}(1-x)^{n_{1}-j}.
\end{equation}
Connecting (\ref{FPRbd}) and (\ref{FNRbd}) and distributions of $y_{i1},y_{i0},z_{i1},z_{i0}$, (\ref{cdfbd1}) and (\ref{cdfbd2}) hold. This completes the proof.
\end{proof}

\subsection*{Proof of Theorem 3}
\begin{proof}
The expression of expected cost on test set $C$ can be derived in the following.
Denoting the cost, true class, predicted class of instance $i$ and $C_i(1,0)$ as $Cost_i, Q_i, \widehat{Q}_i$ and $A$, then we have
\begin{align}
E(Cost_i) &= E(Cost_i|Q_i=0)\pi_0 + E(Cost_i|Q_i=1)\pi_1,\\
E(Cost_i|Q_i=0) &= E(Cost_i|\widehat{Q}_i=1,Q_i=0)P_i(1|0) = \beta A\alpha_{i0},\\
E(Cost_i|Q_i=1)  &= E(Cost_i|\widehat{Q}_i=0,Q_i=1)P_i(0|1)= A\alpha_{i1}.
\end{align}
Therefore
\begin{equation}
C = \sum_{i=1}^{n_{te}}E(Cost_i)= \sum_{i=1}^{n_{te}}(\beta A\alpha_{i0}\pi_0 + A\alpha_{i1}\pi_1) \leq A\sum_{i=1}^{n_{te}}(\beta y_{i0}\pi_0 + y_{i1}\pi_1).
\end{equation}
Then we denote $\beta Ay_{i0}\pi_0$ by $Y_{i0}$, $Ay_{i1}\pi_1$ by $Y_{i1}$ and cumulative distribution function of $y_{i1}, y_{i0}$ by $G_1, G_0$, respectively. The expectation of $y_{i1}$ and $y_{i0}$ can be derived as
\begin{equation}
\begin{aligned}
E(y_{i1})&= \int_0^1 P(y_{i1}>x)dx= \int_0^1(1-G_1(x))dx= \int_0^1\sum_{j=0}^{k_1}\binom{n_1}{j}x^j(1-x)^{n_1-j}dx\\
&= \sum_{j=0}^{k_1}\binom{n_1}{j}\int_0^1x^j(1-x)^{n_1-j}dx= \sum_{j=0}^{k_1}\binom{n_1}{j}Be(j+1,n_1-j+1)\\
&= \sum_{j=0}^{k_1}\frac{n_1!}{(n_1-j)!j!}\cdot \frac{\Gamma(j+1)\Gamma(n_1-j+1)}{\Gamma(n_1+2)}= \frac{k_1+1}{n_1+1},
\end{aligned}
\end{equation}
\begin{equation}
\begin{aligned}\label{yi0.expectation}
E(y_{i0})&= \int_0^1 P(y_{i0}>x)dx= \int_0^1(1-G_0(x))dx= \int_0^1\sum_{j=0}^{n_0-k_0}\binom{n_0}{j}x^j(1-x)^{n_0-j}dx\\
&= \sum_{j=0}^{n_0-k_0}\binom{n_0}{j}\int_0^1x^j(1-x)^{n_1-j}dx= \sum_{j=0}^{n_0-k_0}\binom{n_0}{j}Be(j+1,n_0-j+1)\\
&= \sum_{j=0}^{n_0-k_0}\frac{n_0!}{(n_0-j)!j!}\cdot \frac{\Gamma(j+1)\Gamma(n_0-j+1)}{\Gamma(n_0+2)}= \frac{n_0-k_0+1}{n_0+1},
\end{aligned}
\end{equation}
\begin{equation}
\begin{aligned}
E(y_{i0}^2)&= \int_0^1 x^{2}dG_0(x)= x^{2}G_0(x)\arrowvert_{0}^{1}-2\int_{0}^{1}xG_0(x)dx= 1 - 2\int_{0}^{1}x[1-\overline{G_0}(x)]dx\\
&= 2\int_{0}^{1}x\overline{G_0}(x)dx= 2\sum_{j=0}^{n_0-k_0}\binom{n_0}{j}\int_{0}^{1}x^{j+1}(1-x)^{n_0-j}dx\\
&= 2\sum_{j=0}^{n_0-k_0}\binom{n_0}{j}Be(j+2,n_0-j+1)= 2\sum_{j=0}^{n_0-k_0}\frac{n_0!}{(n_0-j)!j!}\cdot \frac{\Gamma(j+2)\Gamma(n_0-j+1)}{\Gamma(n_0+3)}\\
&= \frac{(n_0-k_0+1)(n_0-k_0+2)}{(n_0+1)(n_0+2)},
\end{aligned}
\end{equation}
\begin{equation}
\begin{aligned}
E(y_{i1}^2)&= \int_0^1 x^{2}dG_1(x)= x^{2}G_1(x)\arrowvert_{0}^{1}-2\int_{0}^{1}xG_1(x)dx= 1 - 2\int_{0}^{1}x[1-\overline{G_1}(x)]dx\\
&= 2\int_{0}^{1}x\overline{G_1}(x)dx= 2\sum_{j=0}^{k_1}\binom{n_1}{j}\int_{0}^{1}x^{j+1}(1-x)^{n_1-j}dx= 2\sum_{j=0}^{k_1}\binom{n_1}{j}Be(j+2,n_1-j+1)\\
&= 2\sum_{j=0}^{k_1}\frac{n_1!}{(n_1-j)!j!}\cdot \frac{\Gamma(j+2)\Gamma(n_1-j+1)}{\Gamma(n_1+3)}= \frac{(k_1+2)(k_1+1)}{(n_1+2)(n_1+1)}.
\end{aligned}
\end{equation}
The variance of $y_{i1},y_{i0}$ can be expressed as
\begin{equation}
\begin{aligned}
Var(y_{i1})&= E(y_{i1}^2) - (Ey_{i1})^2
&= \frac{(k_1+1)(n_1-k_1)}{(n_1+1)^2(n_1+2)},
\end{aligned}
\end{equation}
\begin{equation}
\begin{aligned}
Var(y_{i0})&= E(y_{i0}^2) - (Ey_{i0})^2
&= \frac{k_0(n_0-k_0+1)}{(n_0+1)^2(n_0+2)}.
\end{aligned}
\end{equation}
By Bernstein's inequality \citep{bennett1962probability} we have
\begin{equation}
\begin{aligned}
&P\Big(\sum_{i=1}^{n_{te}}[(Y_{i1}-EY_{i1})+(Y_{i0}-EY_{i0})] > t\sigma\Big)\\
&\leq exp\Bigg\{-\frac{\frac{1}{2}(t\sigma)^2}{\sum\limits_{i=1}^{n_{te}}[E(Y_{i1}-EY_{i1})^2+E(Y_{i0}-EY_{i0})^2]+\frac{M'}{3}t\sigma} \Bigg\} \\
&= exp\Bigg\{-\frac{\frac{1}{2}(t\sigma)^2}{\sum\limits_{i=1}^{n_{te}}[Var(Y_{i1})+Var(Y_{i0})]+\frac{M'}{3}t\sigma} \Bigg\} \\
&= exp\Bigg\{-\frac{(t\sigma)^2}{2\sigma^2+\frac{2M'}{3}t\sigma} \Bigg\}
\leq exp\Bigg\{-\frac{t^2}{2+\frac{2M}{3\sigma}t} \Bigg\} \cong p,
\end{aligned}
\end{equation}
where
\begin{equation}
\begin{aligned}
\sigma^2 &= \sum_{i=1}^{n_{te}}[Var(Y_{i1})+Var(Y_{i0})]= \sum_{i=1}^{n_{te}}[(\pi_1 A)^{2}Var(y_{i1})+(\beta\pi_0 A)^{2}Var(y_{i0})]\\
&= n_{te}A^2 \Big[\pi_1^{2}Var(y_{i1})+(\beta\pi_0)^{2}Var(y_{i0})\Big]\\
&=n_{te}A^2 \Big[\pi_1^2 \frac{(k_1+1)(n_1-k_1)}{(n_1+1)^2(n_1+2)}+(\beta \pi_0)^2 \frac{k_0(n_0-k_0+1)}{(n_0+1)^2(n_0+2)}\Big],
\end{aligned}
\end{equation}
\begin{equation}
\begin{aligned}
M' &= \mathop{max}\limits_{1 \leq i \leq n_{te}}\{|Y_{i1}-EY_{i1}|,|Y_{i0}-EY_{i0}|\} \\
&=A\mathop{max}\limits_{1 \leq i \leq n_{te}}\{\pi_1|y_{i1}-Ey_{i1}|,\beta\pi_0|y_{i0}-Ey_{i0}|\} \\
&\leq A\cdot \mathop{max}\limits_{1 \leq i \leq n_{te}}\{\pi_1|y_{i1}-Ey_{i1}|,\beta\pi_0|y_{i0}-Ey_{i0}|\} \\
&\leq A \cdot max\{\pi_1(|Ey_{i1}|,|1-Ey_{i1}|),\beta\pi_0(|Ey_{i0}|,|1-Ey_{i0}|)\}\\
&= A \cdot max\Big\{\pi_1 max\Big(\frac{k_1+1}{n_1+1},1-\frac{k_1+1}{n_1+1}\Big),\\
&\beta \pi_0 max\Big(\frac{n_0-k_0+1}{n_0+1},1-\frac{n_0-k_0+1}{n_0+1}\Big)\Big\}\cong M.
\end{aligned}
\end{equation}
Then we obtain
\begin{equation}
\begin{aligned}
P(C \leq C^* + t\sigma) &\geq P\Big(\sum_{i=1}^{n_{te}}(Y_{i1}+Y_{i0}) \leq C^* + t\sigma\Big)\\
&= 1-P\Big(\sum_{i=1}^{n_{te}}(Y_{i1}+Y_{i0}) > C^* + t\sigma\Big)\\
&= 1-P\Big(\sum_{i=1}^{n_{te}}[(Y_{i1}-EY_{i1})+(Y_{i0}-EY_{i0})] > t\sigma\Big) \geq 1-p.
\end{aligned}
\end{equation}
This completes the proof.
\end{proof}

\subsection*{Proof of Theorem 4}
\begin{proof}
We list required time for each step of Algorithm 1.
\begin{table*}
\centering
\caption{Main Steps of Algorithm 1 and Their Required Time}
\begin{tabular}{cccc}
\toprule
Step No. &Detail &Time Scale &Coefficient \\
\midrule
1 &calculate scores	&$O(n_v)$ &$d_1$\\
2 &sort 	&$O(n_vlog\,n_v)$ &$d_2 \approx 0$\\
3 &initialize two counters &$O(1)$  &$d_3$\\
4-12 &iterate validation set &$O(n_v)$ &$d_4$ \\
13-15 &find $L^{*}$ and return parameters &$O(1)$ &$d_5$\\
\bottomrule
\end{tabular}
\end{table*}

The total running time can be expressed as
\begin{equation}
\begin{aligned}
Time &= d_{3}n_vlog\,n_v + (d_1 + d_4)n_v + d_2 + d_5\\
&\leq d_{3}n_vlog\,n_v + (d_1 + d_2 + d_4 + d_5)n_v\cong a_1n_vlog\,n_v + a_2n_v,
\end{aligned}
\end{equation}

where $a_1$ is very small. Then we have total time complexity scale as
\begin{equation}
\begin{aligned}
Time\,\,Scale &= O(n_v) + O(n_vlog\,n_v) + O(n_v) + O(1)= O(n_vlog\,n_v).
\end{aligned}
\end{equation}
\end{proof}

\subsection*{Proof of Theorem 5}
\begin{proof}
Let $H_p(n)$ denote the success number of $n$ Bernoulli trials with success probability $p$, then there exists
\begin{equation}
\begin{aligned}
P(y_{i0} > p_{y0}+\varepsilon) &= \sum_{j=0}^{n_0-k_0}\binom{n_{0}}{j}(p_{y0}+\varepsilon)^{j}(1-(p_{y0}+\varepsilon))^{n_{0}-j} \\
&= P(H_{p_{y0}+\varepsilon}(n_0) \leq n_0-k_0),
\end{aligned}
\end{equation}
where $p_{y0} = Ey_{i0} = \frac{n_0-k_0+1}{n_0+1}$. And let $p_{y0}' = p_{y0}+\varepsilon$, $n_0-k_0= (p_{y0}'-\varepsilon_{y0}')n_0 = (p_{y0} + \varepsilon - \varepsilon_{y0}')n_0$, we will obtain
\begin{equation}
\varepsilon_{y0}' = \frac{k_0}{n_0(n_0+1)} + \varepsilon.
\end{equation}
Therefore by Hoeffding's inequality \citep{hoeffding1963probability},
\begin{equation}
\begin{aligned}
P(y_{i0} > p_{y0}+\varepsilon) &= P(H_{p_{y0}'}(n_0) \leq (p_{y0}'-\varepsilon_{y0}')n_0)\leq exp\{-2(\varepsilon_{y0}')^2n_0\}\\
&= exp\Big\{-2\Big[\frac{k_0}{n_0(n_0+1)} + \varepsilon\Big]^2n_0\Big\},
\end{aligned}
\end{equation}
\begin{equation}
P(y_{i0} \leq p_{y0}+\varepsilon) \geq 1- exp\Big\{-2\Big[\frac{k_0}{n_0(n_0+1)} + \varepsilon\Big]^2n_0\Big\}.
\end{equation}
For $z_{i0}$ we have
\begin{equation}
\begin{aligned}
P(z_{i0} < p_{z0}-\varepsilon) &= \sum_{j=n_0-k_0}^{n_0}\binom{n_{0}}{j}(p_{z0}-\varepsilon)^{j}(1-(p_{z0}-\varepsilon))^{n_{0}-j} \\
&= P(H_{p_{z0}-\varepsilon}(n_0) \geq n_0-k_0),
\end{aligned}
\end{equation}
where $p_{z0} = Ez_{i0} = \frac{n_0-k_0}{n_0+1}$. Next we let $p_{z0}' = p_{z0}-\varepsilon$, $n_0-k_0= (p_{z0}'+\varepsilon_{z0}')n_0 = (p_{z0} - \varepsilon + \varepsilon_{z0}')n_0$, we will obtain
\begin{equation}
\varepsilon_{z0}' = \frac{n_0-k_0}{n_0(n_0+1)} + \varepsilon.
\end{equation}
Therefore by Hoeffding's inequality
\begin{equation}
\begin{aligned}
P(z_{i0} < p_{z0}-\varepsilon) &= P(H_{p_{z0}'}(n_0) \geq (p_{z0}'+\varepsilon_{z0}')n_0)\\
&\leq exp\Big\{-2(\varepsilon_{z0}')^2n_0\Big\}= exp\Big\{-2\Big[\frac{n_0-k_0}{n_0(n_0+1)} + \varepsilon\Big]^2n_0\}.
\end{aligned}
\end{equation}
\begin{equation}
P(z_{i0} \geq p_{z0}-\varepsilon) \geq 1 - exp\Big\{-2\Big[\frac{n_0-k_0}{n_0(n_0+1)} + \varepsilon\Big]^2n_0\Big\}.
\end{equation}
Similarly,
\begin{equation}
P(y_{i1} \leq p_{y1}+\varepsilon) \geq 1- exp\Big\{-2\Big[\frac{n_1-k_1}{n_1(n_1+1)} + \varepsilon\Big]^2n_1\Big\},
\end{equation}
\begin{equation}
P(z_{i1} \geq p_{z1}-\varepsilon) \geq 1- exp\Big\{-2\Big[\frac{k_1}{n_1(n_1+1)} + \varepsilon\Big]^2n_1\Big\}.
\end{equation}
where $p_{y1} = Ey_{i1} = \frac{k_1+1}{n_1+1}$, $p_{z1} = Ez_{i1} = \frac{k_1}{n_1+1}$.
And we have
\begin{equation}
\begin{aligned}
&P(y_{i1}\leq p_{y1}+\varepsilon,z_{i1}\geq p_{z1}-\varepsilon)\geq P(y_{i1}\leq p_{y1}+\varepsilon) + P(z_{i1}\geq p_{z1}-\varepsilon)-1\\
&\geq 1- exp\Big\{-2\Big(\frac{k_1}{n_1(n_1+1)} + \varepsilon\Big)^2n_1\Big\} -exp\Big\{-2\Big(\frac{n_1-k_1}{n_1(n_1+1)} + \varepsilon\Big)^2n_1\Big\},
\end{aligned}
\end{equation}
\begin{equation}
\begin{aligned}
&P(y_{i0}\leq p_{y0}+\varepsilon,z_{i0}\geq p_{z0}-\varepsilon)\geq P(y_{i0}\leq p_{y0}+\varepsilon) + P(z_{i0}\geq p_{z0}-\varepsilon)-1\\
&\geq 1- exp\Big\{-2\Big(\frac{n_0-k_0}{n_0(n_0+1)} + \varepsilon\Big)^2n_0\Big\} -exp\Big\{-2\Big(\frac{k_0}{n_0(n_0+1)} + \varepsilon\Big)^2n_0\Big\}.
\end{aligned}
\end{equation}
We denote
\begin{equation}
C_1 = n_{te} A(\pi_1p_{z1}+\beta\pi_0p_{z0}) = n_{te} A\Big[\pi_1\Big(\frac{k_1}{n_1+1}\Big) + \beta \pi_0\Big(\frac{n_0-k_0}{n_0+1}\Big)\Big],
\end{equation}
\begin{equation}
C_2 = n_{te} A(\pi_1p_{y1}+\beta\pi_0p_{y0}) = n_{te} A\Big[\pi_1\Big(\frac{k_1+1}{n_1+1}\Big) + \beta \pi_0\Big(\frac{n_0-k_0+1}{n_0+1}\Big)\Big],
\end{equation}
and
\begin{equation}
C_\varepsilon = n_{te} A(\pi_1+\beta\pi_0)\varepsilon.
\end{equation}
Then for $C = \sum\limits_{i=1}^{n_{te}}(\beta A\alpha_{i0}\pi_0 + A\alpha_{i1}\pi_1)$, Theorem \ref{errbds} and Weierstrass product inequality \citep{wu2005some}
\begin{equation}
\begin{aligned}
P&(C_1 - C_\varepsilon \leq C \leq C_2 + C_\varepsilon)\geq \prod_{i=1}^{n_{te}}P(y_{ij} \leq p_{yj} + \varepsilon,z_{ij} \geq p_{zj} - \varepsilon, j=0,1)\\
&= \prod_{i=1}^{n_{te}}P(y_{i1} \leq p_{y1} + \varepsilon,z_{i1} \geq p_{z1} - \varepsilon)P(y_{i0} \leq p_{y0} + \varepsilon,z_{i0} \geq p_{z0} - \varepsilon)\\
&\geq \prod_{i=1}^{n_{te}}\Big[1- exp\Big\{-2\Big(\frac{k_1}{n_1(n_1+1)} + \varepsilon\Big)^2n_1\Big\} - exp\Big\{-2\Big(\frac{n_1-k_1}{n_1(n_1+1)} + \varepsilon\Big)^2n_1\Big\}\Big]\cdot\\
&\quad\quad\Big[1- exp\Big\{-2\Big(\frac{n_0-k_0}{n_0(n_0+1)} + \varepsilon\Big)^2n_0\Big\}
- exp\Big\{-2\Big(\frac{k_0}{n_0(n_0+1)} + \varepsilon\Big)^2n_0\Big\}\Big]\\
&\geq 1-n_{te}\Big[exp\Big\{-2\Big(\frac{k_1}{n_1(n_1+1)} + \varepsilon\Big)^2n_1\Big\}+ exp\Big\{-2\Big(\frac{n_1-k_1}{n_1(n_1+1)} + \varepsilon\Big)^2n_1\Big\} +\\ &\quad\quad exp\Big\{-2\Big(\frac{n_0-k_0}{n_0(n_0+1)} + \varepsilon\Big)^2n_0\Big\} + exp\Big\{-2\Big(\frac{k_0}{n_0(n_0+1)} + \varepsilon\Big)^2n_0\Big\}\Big]\\
&= 1-O(exp\{-\varepsilon^2n_0\}+exp\{-\varepsilon^2n_1\})= 1- O(exp\{-\varepsilon^2n_v\}).
\end{aligned}
\end{equation}
Similarly, we have
\begin{equation}
\begin{aligned}
P(C \leq C_2 + C_\varepsilon) &\geq 1-n_{te}\Big[exp\Big\{-2\Big(\frac{k_1}{n_1(n_1+1)} + \varepsilon\Big)^2n_1\Big\}+exp\Big\{-2\Big(\frac{n_0-k_0}{n_0(n_0+1)} + \varepsilon\Big)^2n_0\Big\}\Big]\\
&= 1- O(exp\{-\varepsilon^2n_v\}).
\end{aligned}
\end{equation}
\end{proof}

\subsection*{Proof of Theorem 6}
\begin{proof}
Let denote $x^*=\arg\min g(x)$ and $\widehat{x}^*=\arg\min \widehat{g}_{n_v}(x)$. And we first prove some useful lemmas.
\begin{lemma}\label{lem1}
$\min\{\widehat{g}_{n_v}(x)\}=\widehat{g}_{n_v}(\widehat{x}^*) \xrightarrow{p} \min\{g(x)\}=g(x^*)$, as $n_v \longrightarrow \infty$.
\end{lemma}
\begin{proof}[Proof of Lemma 1]
Firstly, for any $\epsilon>0$, by (A2) and $\widehat{g}_{n_v}(\widehat{x}^*) \leq \widehat{g}_{n_v}(x^*)$, we have
\begin{equation}\label{lem.prob1}
P(\widehat{g}_{n_v}(\widehat{x}^*) > g(x^*) + \epsilon) \leq P(\widehat{g}_{n_v}(x^*) > g(x^*) + \epsilon) \longrightarrow 0,
\end{equation}
as $n_v \longrightarrow \infty$. And because (A2) and $g(x^*) \leq g(\widehat{x}^*)$, for any $x^*$, there exists
\begin{equation}
P(\widehat{g}_{n_v}(\widehat{x}^*) < g(x^*)-\epsilon| \widehat{x}^*) \leq P(\widehat{g}_{n_v}(\widehat{x}^*) < g(\widehat{x}^*)-\epsilon| \widehat{x}^*) \longrightarrow 0,
\end{equation}
as $n_v \longrightarrow \infty$ uniformly. Then we have
\begin{equation}\label{lem.prob2}
P(\widehat{g}_{n_v}(\widehat{x}^*) < g(\widehat{x}^*)-\epsilon) = E[P(\widehat{g}_{n_v}(\widehat{x}^*) < g(\widehat{x}^*)-\epsilon| \widehat{x}^*)]\longrightarrow 0.
\end{equation}
Combining (\ref{lem.prob1}) and (\ref{lem.prob2}) together, we have
\begin{equation}
P(|\widehat{g}_{n_v}(\widehat{x}^*) - g(x^*)|> \epsilon) = P(\widehat{g}_{n_v}(\widehat{x}^*) < g(x^*)-\epsilon) + P(\widehat{g}_{n_v}(\widehat{x}^*) > g(x^*) + \epsilon) \longrightarrow 0,
\end{equation}
as $n_v \longrightarrow \infty$, i.e. $\widehat{g}_{n_v}(\widehat{x}^*) \xrightarrow{p} g(x^*)$.
\end{proof}

\begin{lemma}\label{lem2}
$\widehat{x}^* - x^* \xrightarrow{p} 0.$
\end{lemma}
\begin{proof}[Proof of Lemma 2]
By assumption (A2) and Lemma \ref{lem1}, for any $\epsilon>0$, we have
\begin{equation}
P(|g(\widehat{x}^*)-g(x^*)|> \epsilon) \leq P\Big(|\widehat{g}_{n_v}(\widehat{x}^*) - g(\widehat{x}^*)| > \frac{\epsilon}{2}\Big) + P\Big(|\widehat{g}_{n_v}(\widehat{x}^*) - g(x^*)| > \frac{\epsilon}{2}\Big) \longrightarrow 0.
\end{equation}
Then by assumption (A1) and (A4), for any $\epsilon>0$, there exists $\delta>0$ such that
\begin{equation}
|g(\widehat{x}^*) - g(x^*)| > \delta,
\end{equation}
as $|\widehat{x}^* - x^*| > \epsilon$. Then we have
\begin{align}
P(|\widehat{x}^* - x^*|> \epsilon) &\leq P\Big(|g(\widehat{x}^*) - g(x^*)| > \frac{\delta}{2}\Big) + P\Big(|g(\widehat{x}^*) - g(x^*)| \leq \frac{\delta}{2}, |\widehat{x}^* - x^*|> \epsilon\Big)\\
&= P\Big(|g(\widehat{x}^*) - g(x^*)| > \frac{\delta}{2}\Big) \longrightarrow 0,\,\,as\,\,n_v\,\,\longrightarrow\,\,\infty,
\end{align}
i.e. $\widehat{x}^* - x^*  \xrightarrow{p} 0 $.
\end{proof}

\begin{lemma}\label{lem3}
We denote \textit{FPR}, \textit{FNR} for instance $i$ in population for threshold $x^*$ as $p_{i0}$, $p_{i1}$, and \textit{FPR}, \textit{FNR} for instance $i$ in population for threshold $\widehat{x}^*$ as $\alpha_{i0}$, $\alpha_{i1}$. As analysis in Section 3.1, \textit{THORS} chooses the order statistic as the threshold, which implies that \textit{FNR} and \textit{FPR} are both random variables. However, errors corresponding to the optimal threshold leading to minimal expected cost on population are constant. Then we have the following conclusion:
\begin{equation}
\alpha_{i0} \xrightarrow{p} p_{i0},
\alpha_{i1} \xrightarrow{p} p_{i1},\,\,as\,\,n_v\,\,\longrightarrow\,\,\infty.
\end{equation}
\end{lemma}
\begin{proof}[Proof of Lemma 3]
Let $T^0_i$ be instance $i$ of class 0, then we have
\begin{align}
\alpha_{i0} = P(T^0_i > \widehat{x}^*),\\
p_{i0} = P(T^0_i > x^*).
\end{align}
Due to continuity of cumulative function of $T^0$, for any given $\epsilon>0$, there exists $\delta>0$ such that
\begin{equation}
|\alpha_{i0}-p_{i0}| \leq \frac{\epsilon}{2},\,\,when\,\,|\widehat{x}^*-x^*|\leq \delta.
\end{equation}
Then we have
\begin{align}
P(|\alpha_{i0}-p_{i0}|>\epsilon) &\leq P(|\alpha_{i0}-p_{i0}|>\epsilon, |\widehat{x}^*-x^*|\leq \delta) + P(|\widehat{x}^*-x^*|> \delta)\\
&= P(|\widehat{x}^*-x^*|> \delta) \longrightarrow 0,
\end{align}
i.e. $\alpha_{i0} \xrightarrow{p} p_{i0}$. Similarly we can get $\alpha_{i1} \xrightarrow{p} p_{i1}$.
\end{proof}
Now let's return to the original conclusion provided by Theorem \ref{sample.size}. It's easy to see that $\widehat{\alpha}_{i0}=\frac{n_0-k_0}{n_0}$. Then by assumption (A4), for any $\epsilon>0$, there exists
\begin{equation}
P(|\widehat{\alpha}_{i0}-p_{i0}| > \epsilon) \leq P(|\widehat{\alpha}_{i0} - \alpha_{i0}| > \frac{\epsilon}{2}) + P(|p_{i0} - \alpha_{i0}| > \frac{\epsilon}{2}) \longrightarrow 0,
\end{equation}
as $n_v \longrightarrow \infty$, i.e.
\begin{equation}\label{alpha.i0}
\widehat{\alpha}_{i0} \xrightarrow{p} p_{i0}.
\end{equation}
By Hoeffding's Inequality and the fact that $n_0 \sim B(n_v, \pi_0)$, where $\pi_0$ is the marginal probability of class 0, we have
\begin{equation}\label{n_0}
P[n_0 \leq n_v(\pi_0-\epsilon)] \leq e^{-2\epsilon^2n_v}.
\end{equation}
Combining (\ref{alpha.i0}) and (\ref{n_0}), for any $\epsilon>0$, we have
\begin{align}
P\Big(\Big|\frac{n_0-k_0}{n_0+1}-p_{i0}\Big|>\epsilon\Big) &\leq P\Big(\Big|\frac{n_0-k_0}{n_0+1}-\frac{n_0-k_0}{n_0}\Big|> \frac{\epsilon}{2}\Big) + P\Big(\Big|p_{i0}-\widehat{\alpha}_{i0}\Big|> \frac{\epsilon}{2}\Big)\\
&\leq P\Big(n_0 \leq \frac{2}{\epsilon}-1\Big) + P\Big(\Big|p_{i0}-\widehat{\alpha}_{i0}\Big|> \frac{\epsilon}{2}\Big)\\
&\leq P[n_0 \leq (\pi_0 - \epsilon)n_v]+ P\Big(\Big|p_{i0}-\widehat{\alpha}_{i0}\Big|> \frac{\epsilon}{2}\Big)\\
&\leq e^{-2\epsilon^2n_v}+ P\Big(\Big|p_{i0}-\widehat{\alpha}_{i0}\Big|> \frac{\epsilon}{2}\Big) \longrightarrow 0,
\end{align}
i.e. $\frac{n_0-k_0}{n_0+1} \xrightarrow{p} p_{i0}$ as $n_v \longrightarrow \infty$. Similarly, we also have $\frac{k_1}{n_1+1} \xrightarrow{p} p_{i1}$. This completes the proof.

\end{proof}

\end{document}